%% file: target_prop.tex
\documentclass{article} 
\usepackage{times}

\input{packages}

\input{shortcuts}
\input{notations}
\usepackage{natbib}
\setcitestyle{authoryear,round,citesep={;},aysep={,},yysep={;}}
\usepackage{fullpage}

\title{Target Propagation via Regularized Inversion}

\author{Vincent Roulet \& Zaid Harchaoui  \\
Department of Statistics\\
University of Washington
}

\begin{document}
\maketitle

\begin{abstract}
	Target Propagation (TP) algorithms compute targets instead of gradients along neural networks and propagate them backward in a way that is similar yet different than gradient back-propagation (BP).
The idea was first presented as a perturbative alternative to back-propagation that may improve gradient evaluation accuracy when training multi-layer neural networks~\citep{lecun1989gemini}. However, TP may have remained more of a template algorithm with many variations than a well-identified algorithm. Revisiting insights of~\citet{lecun1989gemini} and more recently of~\citet{lee2015difference}, we present a simple version of target propagation based on a regularized inversion of network layers, easily implementable in a differentiable programming framework. We compare its computational complexity to the one of BP and delineate the regimes in which TP can be attractive compared to BP. We show how our TP can be used to train recurrent neural networks with long sequences on various sequence modeling problems. The experimental results underscore the importance of regularization in TP in practice.
\end{abstract}

\section{Introduction}
\input{sections/01_intro}

\section{Target Propagation with Linearized Regularized Inverses}\label{sec:target_prop}
\input{sections/02_setting}

\section{Gradient Back-propagation versus Target Propagation}\label{sec:graph}
\input{sections/03_cplxity}

\section{Experiments}\label{sec:exp}
\input{sections/04_exp}

\paragraph{Acknowledgments}
This work was supported by NSF CCF-1740551, NSF DMS-1839371, the CIFAR program ``Learning in Machines and Brains'', and faculty research awards. We thank Nikolay Manchev for  all the details he provided on his code. 
\bibliography{target_prop_refs}
\bibliographystyle{plainnat}

\clearpage
\appendix
\section*{Appendix Plan}
The Appendix is organized as follows.
\begin{enumerate}[nosep]
	\item Sec.~\ref{app:rnn} recalls how gradient back-propagation works for RNNs.
	\item Sec.~\ref{app:target_prop_algo} details the implementations of target propagation.
	\item Sec.~\ref{app:proofs} details the differences between TP and gradient back-propagation or Gauss-Newton optimization.
	\item Sec.~\ref{app:exp_details} details the parameters used in our experiments and presents additional experiments.
\end{enumerate}

\section{Gradient back-propagation in Recurrent Neural Networks}\label{app:rnn}
\input{appendix/a_rnn}

\section{Detailed Implementation}\label{app:target_prop_algo}
\input{appendix/b_target_prop_algo}

\section{Target Propagation vs Gradient or Gauss-Newton descent}\label{app:proofs}
\input{appendix/c_proofs}

\section{Experimental Details}\label{app:exp_details}
\input{appendix/d_exp_details}

\end{document}

%% file: packages.tex
\usepackage[utf8]{inputenc}
\usepackage[T1]{fontenc} 
\usepackage{mathtools}
\usepackage{enumitem}
\usepackage{booktabs}
\usepackage{algpseudocode}
\usepackage{algorithm, algorithmicx}
\usepackage{bbold}
\usepackage{subcaption}
\usepackage{makecell}
\usepackage{accents}
\usepackage{amsthm}
\usepackage{thm-restate}
\usepackage{booktabs}

\let\originalparagraph\paragraph
\renewcommand{\paragraph}[2][.]{\originalparagraph{#2#1}}

\usepackage{xcolor}
\usepackage{color}
\definecolor{puorange}{rgb}{0.80,0.20,0}
\definecolor{bluegray}{RGB}{65,105,225}
\definecolor{greengray}{rgb}{0.05,0.50,0.15}
\definecolor{darkbrown}{rgb}{0.40,0.2,0.05}
\definecolor{darkcyan}{RGB}{0,160,200}
\definecolor{black}{rgb}{0,0,0}
\definecolor{darkgray}{rgb}{0.3,0.3,0.3}
\definecolor{darkmagenta}{RGB}{139,0,139}
\definecolor{darkpurple}{RGB}{128,0,128}
\definecolor{darkred}{RGB}{90,0,0}
\definecolor{nicered}{RGB}{178,34,34}
\definecolor{darkorange}{RGB}{255,90,0}
\definecolor{darkblue}{RGB}{20,20,100}
\definecolor{verydarkorange}{RGB}{180,50,0}
\definecolor{teal}{RGB}{0,128,128}
\definecolor{indigo}{RGB}{75,0,130}
\usepackage[colorlinks,citecolor=bluegray,linkcolor=darkred,urlcolor=darkcyan,breaklinks]{hyperref}

\newcommand{\mediumblue}[1]{\textcolor{bluegray}{#1}}

\newcommand{\mediumred}[1]{\textcolor{nicered}{#1}}

\newcommand{\orange}[1]{\textcolor{darkorange}{#1}}

\newcommand{\violet}[1]{\textcolor{darkpurple}{#1}}
\newcommand{\darkviolet}[1]{\textcolor{indigo}{#1}}

\newtheorem{theorem}{Theorem}[section]

\newtheorem{lemma}[theorem]{Lemma}
\newtheorem{corollary}[theorem]{Corollary}

\usepackage{caption}

%% file: shortcuts.tex
\newcommand{\reals}{{\mathbb R}}

\newcommand{\idm}{\operatorname{I}}

\newcommand{\diag}{\operatorname{ diag}}

\DeclareMathOperator*{\argmin}{argmin}

\def\shortdisplay{\setlength{\abovedisplayskip}{5pt}%
	\setlength{\belowdisplayskip}{5pt}%
	\setlength{\abovedisplayshortskip}{2pt}%
	\setlength{\belowdisplayshortskip}{2pt}}

\let\oldselectfont\selectfont
\def\selectfont{\oldselectfont\shortdisplay}

%% file: notations.tex
\newcommand{\var}{x}

\newcommand{\dyn}{f}
\newcommand{\horizon}{\tau}

\newcommand{\chain}{g}
\newcommand{\dyninv}{\psi}

\newcommand{\stepsize}{\gamma}
\newcommand{\target}{v}

\newcommand{\inpt}{x}
\newcommand{\outpt}{y}
\newcommand{\hidden}{h}

\newcommand{\activ}{a}

\newcommand{\dimhidden}{p}
\newcommand{\diminpt}{d}

\newcommand{\softmax}{s}

\newcommand{\loss}{\ell}

\newcommand{\nxt}{\operatorname{next}}

\newcommand{\disp}{\lambda}
\newcommand{\pred}{c}
\newcommand{\param}{\theta}
\newcommand{\descent}{d}
\newcommand{\reg}{r}

\newcommand{\invsign}{-1}

\newcommand{\dimparam}{q}
\newcommand{\invsimp}{g}
\newcommand{\inter}{u}
\newcommand{\TP}{TP~}
\newcommand{\TPP}{TP}

%% file: sections/01_intro.tex
Target propagation algorithms can be seen as perturbative learning alternatives to the gradient back-propagation algorithm, where virtual targets are propagated backward instead of gradients~\citep{lecun1986learning, lecun1989gemini, rohwer1990moving,mirowski2009dynamic,bengio2014auto, goodfellow2016deep}. A high-level summary is presented in Fig.~\ref{fig:target_prop_scheme}: while gradient back-propagation considers storing intermediate gradients in a forward pass,  target propagation algorithms proceed by computing and storing approximate inverses. The approximate inverses are then passed on backward along the graph of computations to finally yield a weight update for stochastic learning. 

Target propagation aims to take advantage of the availability of approximate inverses to compute better descent directions for the objective at hand.~\citet{bengio2013estimating, bengio2020deriving} argued that the approach could be relevant for problems involving multiple compositions such as the training of Recurrent Neural Networks (RNNs), which generally suffer from the phenomenon of exploding or vanishing gradients~\citep{hochreiter1998vanishing, bengio1994learning, schmidhuber1992learning}.
Recently, empirical results indeed showed the potential advantages of target propagation over classical gradient back-propagation for training RNNs on several tasks~\citep{manchev2020target}. However, these recent investigations remain built on multiple approximations, which hinder the analysis of the core idea of TP, i.e., using layer inverses. 

On the theoretical side, difference target propagation, a modern variant of target propagation, was related to an approximate Gauss-Newton method, suggesting interesting venues to explain the benefits of target propagation~\citep{bengio2020deriving,  meulemans2020theoretical}. Previous works have considered approximating inverses by adding multiple reverse layers~\citep{manchev2020target, meulemans2020theoretical, bengio2020deriving}. However, it is unclear whether such reverse layers actually learn layer inverses during the training process. Even if they were, the additional cost of computational complexity of learning approximate inverses should be carefully accounted for.

In this work, we propose a simple target propagation approach, revisiting the original insights of~\citet{lecun1989gemini} on the critical importance of the good conditioning of layer inverses. We define regularized inverses through a variational formulation and we obtain approximate inverses via these regularized inverses. In this spirit, we can also interpret the difference target propagation formula~\citep{lee2015difference} as a finite difference approximation of a linearized regularized inverse. We propose a smoother formula that can directly be integrated into a differentiable programming framework.

We detail the computational complexity of the proposed target propagation and compare it to the one of gradient back-propagation, showing that the additional cost of computing inverses can be effectively amortized for very long sequences. Following the benchmark of~\citet{manchev2020target}, we observe that the proposed target propagation can perform better than classical gradient-based methods on several tasks involving RNNs. 

The code to reproduce the experiments is provided at {\small \url{https://github.com/vroulet/tpri}}. The appendix details the implementations of the algorithms and discuss previous interpretations of target propagation as a Gauss-Newton method. 

\paragraph{Related work}
Many variations of back-propagation algorithms have been explored; see~\cite{werbos1994roots,goodfellow2016deep} for an extensive bibliography. Closer to target propagation, penalized formulations of the training problem have been considered to decouple the optimization of the weights in a distributed way or using an ADMM approach~\citep{carreira2014distributed,taylor2016training, gotmare2018decoupling}. Rather than modifying the backward operations in the layers, one can also modify the weight updates for deep forward networks by using a regularized inverse~\citep{frerix2018proximal}.~\citet{wiseman2017training} recast target propagation as an ADMM-like algorithm for language modeling and reported disappointing experimental results. Recently, in a careful experimental benchmark evaluation,~\citet{manchev2020target} explored further target propagation to train RNNs, mapping a sequence to a single final output, in an attempt to understand the benefits of target propagation to capture long-range dependencies, and obtained promising experimental results. Another line of research has considered synthetic gradients that approximate gradients using an additional layer instead of using back-propagated gradients~\citep{jaderberg2017decoupled, czarnecki2017understanding} to speed up the training of deep neural networks.
Recently, \citet{ahmad2020gait, dalm2021scaling}  considered using  analytical inverses to implement target propagation and blend it with what they called a gradient-adjusted incremental formula. Yet, an additional orthogonality penalty is critical for their approach to work.
Recently, \citet{meulemans2020theoretical} considered using  as many  reverse layers as  forwarding operations.  We focus here on the optimization gains of using target propagation that cannot be obtained by adding a prohibitive number of reverse layers.
Finally, we do not discuss the biological plausibility of TP since we are unable to comment on this. We refer the interested reader to, e.g.,~\citep{bengio2020deriving}.

\begin{figure}
	\begin{center}
		\includegraphics[width=0.9\linewidth]{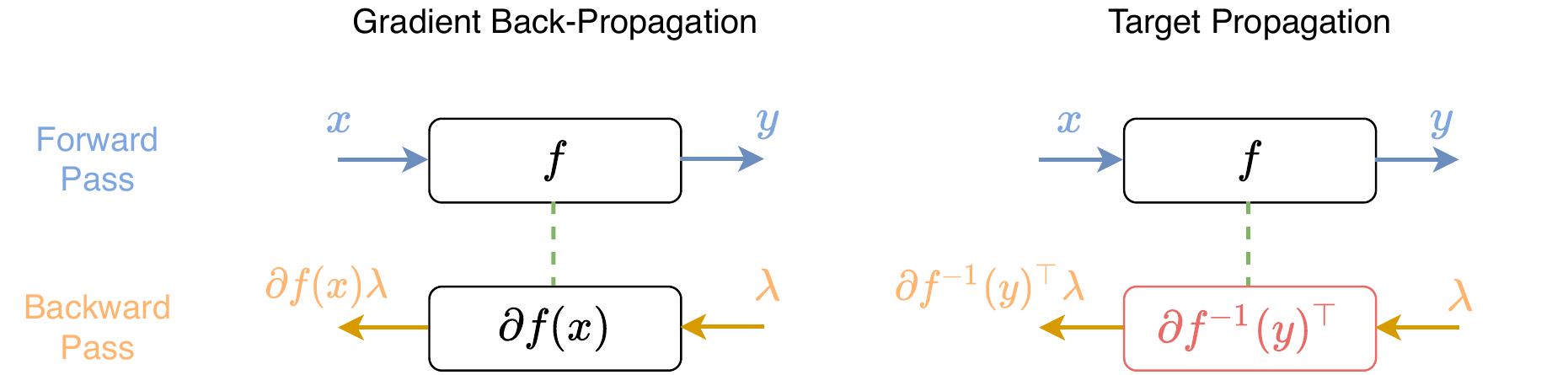}
		\caption{\small Our implementation of target propagation uses linearization of gradient  inverses instead of gradients in a backward pass akin to gradient back-propagation.\label{fig:target_prop_scheme}}
	\end{center}
\end{figure}

\paragraph{Notations}
For  $f:	\reals^p\times \reals^q \rightarrow \reals^d$, 
we  denote  
the partial derivative of $f$ w.r.t. $x \in \reals^p$ on a point $(x, y)\in \reals^p \times \reals^q$ as
$
\partial_x f(x, y) = \left({\partial f^j(x, y)}/{\partial x_i} \right)_{\substack{i, j}} \in \reals^{d\times p}.
$

%% file: sections/02_setting.tex
While target propagation was initially developed for multi-layer neural networks, we focus on its implementation for recurrent neural networks, as we shall follow the benchmark of~\citet{manchev2020target} in the experiments. Recurrent Neural Networks (RNNs) are also a canonical family of neural networks in which interesting phenomena arise in back-propagation algorithms. 

\paragraph{Problem setting}
A simple RNN  parameterized by $\param = (W_{\hidden\hidden}, W_{\inpt\hidden}, b_\hidden, W_{\hidden\outpt}, b_\outpt)$ maps a sequence of inputs $\inpt_{1:\horizon} = (\inpt_1, \ldots, \inpt_\horizon)$ to an output $\hat \outpt = \chain_{\param}( \inpt_{1:\horizon})$ by computing hidden states $\hidden_t\in \reals^\dimhidden$ corresponding to the inputs $x_t$. 

Formally, the output $\hat y$ and the hidden states $\hidden_t$ are computed as an output operation following transition operations defined as
\begin{align}
	\hat \outpt & =  \pred_\param(\hidden_\horizon) := \softmax(W_{\hidden \outpt} \hidden_\horizon +  b_\outpt),  \nonumber\\
	\hidden_t & = \dyn_{\param, t}(\hidden_{t-1})
	:= \activ(W_{ \inpt\hidden} x_t + W_{\hidden\hidden} \hidden_{t-1} + b_\hidden) \quad \mbox{for} \ t\in \{1, \ldots, \horizon\},\nonumber
\end{align}
where $\softmax$ is, e.g., the soft-max function for classification tasks,  $\activ$ is a non-linear operation such as the hyperbolic tangent  function, and the initial hidden state is generally fixed as $\hidden_0 = 0$. 
Given samples of sequence-output pairs ($\inpt_{1:\horizon}, \outpt)$, the RNN is trained to minimize the error $\loss(\outpt, \chain_{\param}( \inpt_{1:\horizon}))$ of predicting $\hat \outpt=\chain_{\param}( \inpt_{1:\horizon})$ instead of $\outpt$.

As one considers longer sequences, RNNs face the challenge of exploding/vanishing gradients ${\partial \chain_{\param}(\inpt_{1:\horizon})}/{\partial \hidden_t}$~\citep{bengio1995diffusion}; see Appendix~\ref{app:rnn} for more discussion. We acknowledge that specific parameterization-based strategies have been proposed to address this issue of exploding/vanishing gradients, such as orthonormal parameterizations of the weights~\citep{arjovsky2016unitary, helfrich2018orthogonal, lezcano2019cheap}. 
The focus here is to simplify and understand target propagation as a backpropagation-type algorithm using RNNs as a workbench. Indeed, training RNNs is an optimization problem involving multiple compositions for which approximate inverses can easily  be available. The framework could also be potentially applied to, {e.g.}, time-series or control models~\citep{roulet2019iterative}. 

Given the parameters $W_{\hidden\hidden}, W_{\inpt \hidden}, b_\hidden$ of the transition operations, we can get approximate inverses of $\dyn_{\param, t}(\hidden_{t-1})$ for all $t\in \{1, \ldots, \horizon\}$,  that yield optimization surrogates that can be better performing than the ones corresponding to regular gradients. We present below a \textit{simple version} of target propagation based on \textit{regularized inverses} and \textit{inverse linearizations}.

\paragraph{Back-propagating targets}
The  idea of target propagation is to compute virtual targets $\target_t$ for each layer $t=\horizon, \ldots, 1$ such that if the layers were able to match their corresponding target at time $t$, i.e., $ \dyn_{\param, t}(\hidden_{t-1}) \approx \target_t$, the  objective  would decrease. The final target $\target_\horizon$ is computed  as a  gradient step on the loss w.r.t. $\hidden_\horizon$.  The targets are then back-propagated using an approximate inverse\footnote{In the following, to ease the presentation, we  abuse notations and denote approximate inverses  by $\dyn_{\param, t}^{\invsign}$.} $\dyn_{\param, t}^{\invsign}$ of $\dyn_{\param, t}$ at each time step.

Formally, consider an RNN that computed $\horizon$ states $\hidden_1, \ldots, \hidden_\horizon$ from a sequence $\inpt_1, \ldots, \inpt_\horizon$ with associated output $y$. For a given stepsize $\stepsize_\hidden>0$, we propose to back-propagate targets by computing
\begin{align}\label{eq:final_targ}
	\target_\horizon & = \hidden_\horizon  -
	\stepsize_\hidden 
	{\partial_\hidden \loss(y, \pred_\param(\hidden_\horizon))}, 
	\\
	\target_{t-1} &= \hidden_{t-1} + 	{\partial_{\hidden}\dyn_{\param, t}^{-1}(\hidden_t) }^\top  (\target_t - \hidden_t), \quad \mbox{for}\  t \in \{\horizon, \ldots, 1\}. \label{eq:newtarg}
\end{align}
The update rule~\eqref{eq:newtarg} blends two ideas: i) regularized inversion; ii) linear approximation. We shall describe below that our update~\eqref{eq:newtarg} allows us to interpret the ``magic formula'' of difference target propagation in Eq.~15 of~\citet{lee2015difference} as $0$th-order finite difference approximation, while ours is a $1$st-order linear approximation. We shall also show that~\eqref{eq:newtarg} puts in practice an insight from~\citet{bengio2020deriving} suggesting to use the inverse of the gradients in the spirit of  a Gauss-Newton method.

Once all targets are computed, the parameters of the transition operations are updated such that the outputs of $\dyn_{\param, t}$ at each time step move closer to the given target. Formally,  the update consists of a gradient step with stepsize $\stepsize_\theta$ on the squared error between the targets and the current outputs, i.e., for $\param_\hidden \in \{W_{\hidden\hidden}, W_{\inpt\hidden}, b_\hidden\}$,
\begin{align}\label{eq:shoot}
	\param_\hidden^{\nxt} & = \param_\hidden - \stepsize_{\param} \sum_{t=1}^\horizon
	{\partial_{\param_\hidden}\| \dyn_{\param, t}(\hidden_{t-1}) - \target_t\|_2^2/2}.
\end{align}
As for the parameters $\param_\outpt = (W_{ \hidden \outpt}, b_\outpt)$ of the output operation, they are updated by a simple gradient step  on the loss with a stepsize $\stepsize_{\param}$.  

\subsection{Regularized Inversion}
To explore further the original idea of~\citet{lecun1989gemini}, we consider using the variational definition of the inverse,
\begin{align}
	\dyn_{\param, t}^{-1}(\target_t) & = \argmin_{\target_{t-1}\in \reals^\dimhidden} \|\dyn_{\param, t}(\target_{t-1}) - \target_t\|_2^2    =\argmin_{\target_{t-1}\in \reals^\dimhidden} \|\activ(W_{\inpt \hidden} \inpt_t + W_{\hidden\hidden}\target_{t-1} + b_\hidden ) - \target_t\|_2^2.    \label{eq:real_inv}
\end{align}
As long as $\target_t$ belongs to the image $\dyn_{\param, t}(\reals^\dimhidden)$ of  $\dyn_{\param, t} $, this definition recovers exactly the inverse of $\target_t$ by $\dyn_{\param, t}$. More generally, if $\target_t \not \in \dyn_{\param, t}(\reals^\dimhidden) $,  Eq.~\eqref{eq:real_inv} computes the \emph{best approximation} of the inverse in the sense of the Euclidean projection. When one considers an activation function $a$ and $\theta_\hidden = (W_{\hidden\hidden}, W_{\inpt\hidden}, b_\hidden)$, the solution of~\eqref{eq:real_inv} can easily be computed.

Formally, for the sigmoid, the hyperbolic tangent or the ReLU, their inverse can be obtained analytically for any $\target_t \in \activ(\reals^\dimhidden)$. 
So for $\target_t \in \activ(\reals^\dimhidden)$ and  $W_{\hidden\hidden}$ full rank, we get
\begin{align*}
	\dyn_{\param, t}^{-1}(\target_t) & = (W_{\hidden\hidden}^\top W_{\hidden\hidden})^{-1}W_{\hidden \hidden}^\top( \activ^{-1}(\target_t) -W_{\inpt \hidden} \inpt_t  - b_\hidden).
\end{align*}

If $\target_t \not \in \activ(\reals^\dimhidden)$, the minimizer of~\eqref{eq:real_inv} is obtained by first projecting $\target_t$ onto $\activ(\reals^\dimhidden)$, before inverting the linear operation. 
To account for non-invertible matrices $W_{\hidden\hidden}$, we also add a regularization  in the computation of the inverse. Overall we consider approximating the inverse of the layer by a regularized inverse of the form
\begin{equation*}
	\dyn_{\param, t}^{\invsign}(\target_t) = (W_{\hidden\hidden}^\top W_{\hidden\hidden} +\reg \idm)^{-1}W_{\hidden\hidden}^\top( \activ^{-1}(\pi(\target_t)) -W_{\inpt\hidden} \inpt_t  - b_\hidden),
\end{equation*}
with $\reg>0$ and $\pi$ a projection onto $a(\reals^\dimhidden)$.

\paragraph{Regularized inversion vs. parameterized inversion}
\citet{bengio2014auto, manchev2020target} parameterize the inverse as a reverse layer such that
\begin{equation*}
	\dyn_{\param, t}^{\invsign}(\target_t)  = \dyninv_{\param', t}(\target_t) := \activ(W_{ \inpt \hidden} \inpt_t + V\target_t  + c),
\end{equation*}
and learn the parameters $\param' = (V, c)$ for this reverse layer to approximate the inverse of the forward computations. 
The parameterized layer needs to be learned  to get a good approximation which involves numerically solving an optimization problem for each layer. These optimization problems come with a computational cost that can be better controlled by using regularized inversions presented earlier.

However, the approach based on parameterized inverses may lack theoretical grounding, as pointed out by~\citet{bengio2020deriving}, as we do not know how close the learned inverse is to the actual inverse throughout the training process. In contrast, the regularized inversion~\eqref{eq:real_inv} is less \textit{ad hoc} and clearly defined and, as we shall show in the experiments, leads to competitive performance on real datasets.

In any case, the analytic formulation of the inverse gives simple insights on an approach with parameterized inverses. Namely, the analytical formula suggests parameterizing the \emph{reverse layer} s.t.
(i) the reverse activation is defined as the inverse of the activation and not any activation, (ii)  the layer uses a non-linear operation followed by a linear one instead of the usual scheme, i.e., a linear operation followed by a non-linear one.

\subsection{Linearized Inversion}
Earlier instances of target propagation used direct inverses of the network layers such that the target propagation update formula would read $\target_{t-1} = \dyn_{\param, t}^{-1}(\target_t)$ in~\eqref{eq:newtarg}. Yet, we are unaware of a successful implementation of TP using directly the inverses.  To circumvent this issue, \citet{lee2015difference}  proposed the \emph{difference target propagation} formula that back-propagates the targets as
\[
\target_{t-1} = \hidden_{t-1} +  \dyn_{\param, t}^{-1}(\target_t) - \dyn_{\param, t}^{-1}(\hidden_t).
\]
If the inverses were exact, the difference target propagation formula would reduce to $\target_{t-1} = \dyn_{\param, t}^{-1}(\target_t)$. \citet{lee2015difference} introduced the difference target propagation formula to mitigate the approximation error of the inverses by parameterized layers.
The difference target propagation formula can naturally be interpreted as an approximation of the linearization used in~\eqref{eq:newtarg}, as
\begin{equation}\nonumber
	\dyn_{\param, t}^{-1}(\target_t) - \dyn_{\param, t}^{-1}(\hidden_t)
	=
	{\partial_{\hidden}\dyn_{\param, t}^{-1}(\hidden_t) }^\top  (\target_t - \hidden_t)
	+ O(\|\target_t-\hidden_t\|_2^2),
\end{equation}
where ${\partial_{\hidden}\dyn_{\param, t}^{-1}(\hidden_t) }^\top$ denotes the Jacobian of the inverse of the layer at $\hidden_t$.

We show in Appendix~\ref{app:exp_details} that the first-order approximation we propose~\eqref{eq:newtarg}
leads to slightly better training curves than the finite-difference approximation. 
Moreover, our interpretation illuminates the ``mystery'' of this formula, which appeared to be critical to the success of target propagation.

%% file: sections/03_cplxity.tex
\paragraph{Graph of computations}
Gradient back-propagation and target propagation both compute a descent direction for the objective at hand. The  difference lies in the oracles computed and stored in the forward pass, while the graph of computations remains the same. To clarify this view, we reformulate target propagation  in terms of displacements
$\disp_t := \target_t - \hidden_t$ such that Eq.~\eqref{eq:final_targ}, ~\eqref{eq:newtarg} and~\eqref{eq:shoot} read
\begin{align*}
	\disp_\horizon  & = -
	\stepsize_\hidden 
	{\partial_\hidden \loss(y, \pred_\param(\hidden_\horizon))},
	\qquad  \quad
	\disp_{t-1}  =   {\partial_{\hidden}\dyn_{\param, t}^{-1}(\hidden_t) }^\top \disp_t, \quad \mbox{for}\  t \in \{\horizon, \ldots, 1\}, 
	\\
	\descent_{\param_\hidden}& =  \sum_{t=1}^\horizon  
	{\partial_{\param_{\hidden}} \dyn_{\param, t} (\hidden_{t-1})} {{ \disp_t}}, 
	\qquad  \quad \hspace{4pt}
	\param_\hidden^{\nxt} =  \param_\hidden + \stepsize_\hidden \descent_{\param_\hidden}.
\end{align*}
Target propagation amounts then  to computing a descent direction $\descent_{\param_\hidden}$ for the parameters~$\param_\hidden$ with a graph of computations, illustrated in Fig.~\ref{fig:target_prop}, analogous to that of gradient-back-propagation illustrated in Appendix~\ref{app:rnn}. 
The difference lies in the use of the Jacobian of the inverse
\[
{\partial_{\hidden}\dyn_{\param, t}^{-1}(\hidden_t) }^\top
\quad 
\mbox{instead of} \quad \partial_{\hidden} \dyn_{\param, t}(\hidden_{t-1}).
\]
The implementation of TP with the formula~\eqref{eq:newtarg} can be done in a differentiable programming framework, where, rather than computing the gradient of the layer, one evaluates the inverse and keep the Jacobian of the inverse.  With the precise graph of computation of TP and BP, we can compare  their computational complexity explicitly and bound the difference of the directions they output.

\begin{figure*}
	\begin{center}
		\includegraphics[width=0.8\linewidth]{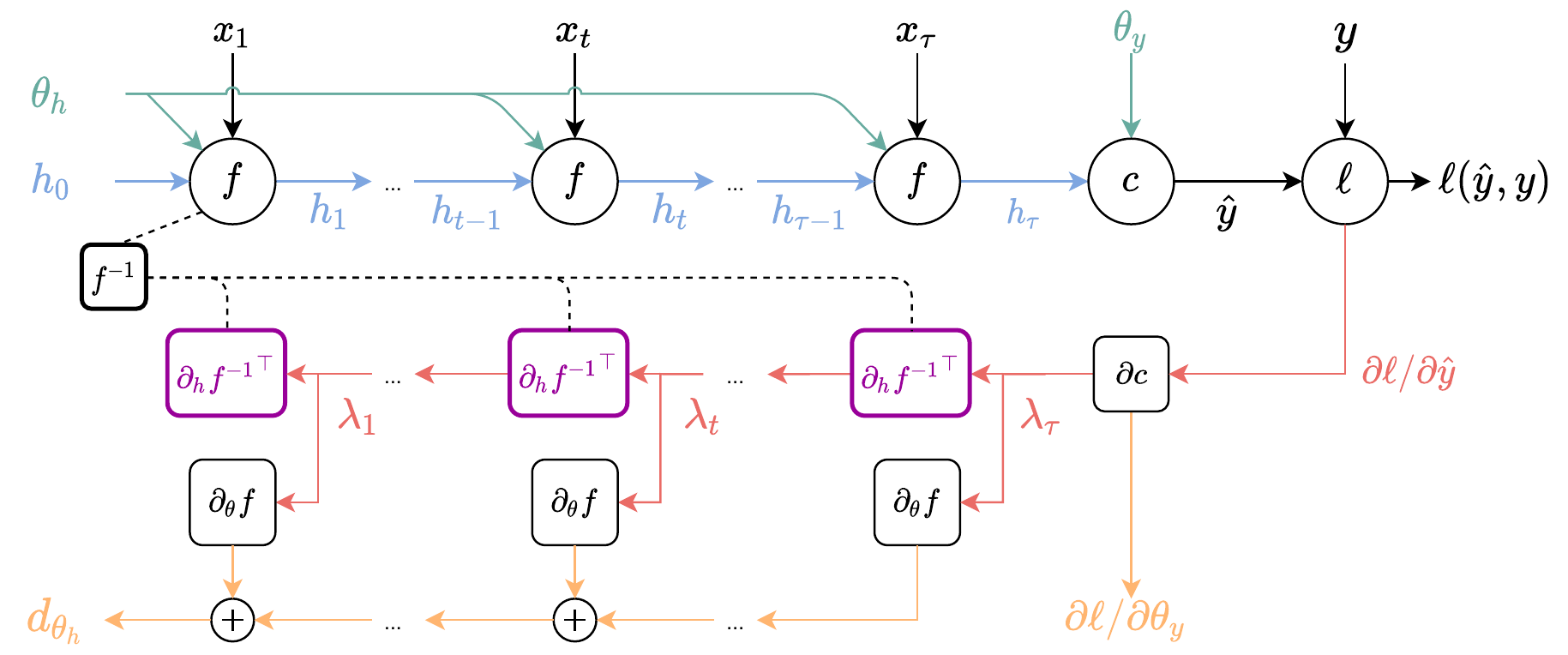}
	\end{center}
	\caption{\small The graph of computations of  target propagation is the same as the one of gradient back-propagation except that $f^{-1}$ needs to be computed and Jacobian of the inverses, $\partial_\hidden \dyn^{-1} ~^\top$  are used instead of gradients $\partial_{\hidden} \dyn$ in the transition operations.\label{fig:target_prop}}
\end{figure*}

\paragraph{Arithmetic complexity}
Clearly, the space complexities of gradient back-propagation (BP) and our implementation of  target propagation (TP) are the same since the Jacobians of the inverse, and the original gradients have the same size.
In terms of time complexity, TP appears at first glance to introduce an important overhead since it  requires the computation of some inverses. However, a close inspection of the formula of the regularized inverse reveals that a matrix inversion needs to be computed only once for all time steps. Therefore the cost of the inversion may be amortized if the length of the sequence is particularly long.  

\vspace{3em}

Formally, the time complexity of the forward-backward pass of gradient back-propagation is essentially driven by matrix-vector products, i.e., 
\begin{align*}
	\mathcal{T}_{\textrm{BP}} & = \sum_{t=1}^{\horizon} \left[
	\underbrace{\mathcal{T}(\dyn_{\param, t}) 
		+ \mathcal{T}\left({\partial_\hidden \dyn_{\param, t}}\right)
		+ \mathcal{T}\left({\partial_{\param_\hidden} \dyn_{\param, t}}\right)}_{\textrm{Forward}} 
	+ \underbrace{\mathcal{T}(\partial_\hidden \dyn_{\param, t}(\hidden_{t-1}))
		+ \mathcal{T}(\partial_\param \dyn_{\param, t}(\hidden_{t-1})) }_{\textrm{Backward}}
	\right] \\
	& \approx \horizon( \diminpt\dimhidden + \dimhidden^2 +   \dimhidden\dimparam) + \horizon(\dimhidden^2 + \dimhidden\dimparam),
\end{align*}
where $\diminpt$ is the dimension of the input $x_t$, $\dimparam$ is the dimension of the parameters $\param_\hidden$, for a function $f$ we denote by $\mathcal{T}(f)$ the time complexity to evaluate $f$ and we consider e.g. $\partial_\param \dyn_{\param, t}(\hidden_{t-1}))$ as the linear function $\lambda \rightarrow \partial_\param \dyn_{\param, t}(\hidden_{t-1}))\lambda$. 

On the other hand, the time complexity of target propagation is
\begin{align*}
	\mathcal{T}_{\textrm{TP}} & =\sum_{t=1}^{\horizon} \left[
	\underbrace{	\mathcal{T}(\dyn_{\param, t})  {+} 	\mathcal{T}(\dyn_{\param, t}^{-1}) 
		{+} \mathcal{T}\left({\partial_{\param_\hidden} \dyn_{\param, t}}\right)  {+} \mathcal{T}(\partial_{\hidden}\dyn_{\param,t}^{-1})}_{\textrm{Forward}}
	{+} \underbrace{\mathcal{T}(\partial_{\hidden}\dyn_{\param,t}^{-1})(\hidden_t)^\top )
		{+} \mathcal{T}(\partial_\param \dyn_{\param, t}(\hidden_{t-1}) )}_{\textrm{Backward}}
	\right] \\
	& \quad + \mathcal{P}(\dyn_{\param, t}^{-1}),
\end{align*}
where $ \mathcal{P}(\dyn_{\param, t}^{-1}) $ is the cost of encoding the inverse, which, in our case, amounts to the cost of encoding
$
\invsimp_\param: z \rightarrow (W_{\hidden\hidden}^\top  W_{\hidden\hidden} + \reg \idm)^{-1} W_{\hidden\hidden}^\top,
$
such that our regularized inverse can be computed as $\dyn_{\param, t}^{-1}(\target_t) = \invsimp_\param(\activ^{-1}(\target_t) - W_{\inpt\hidden}\inpt_t + b_\hidden)$. Encoding $\invsimp$ comes at the cost of inverting one matrix of size $\dimhidden$. Therefore, the time-complexity of target propagation can be estimated as 
\begin{align*}
	\mathcal{T}_{\textrm{TP}} &  \approx \dimhidden^3 +\horizon( \diminpt\dimhidden + \dimhidden^2 +   \dimhidden\dimparam) + \horizon(\dimhidden^2 + \dimhidden\dimparam) \approx \mathcal{T}_{\textrm{BP}} \quad \mbox{if} \ \horizon \geq \dimhidden,
\end{align*}
i.e., for long sequences whose length is larger than the dimension of the hidden states, the cost of TP with regularized inverses is approximately the same as the cost of BP. 
If a parameterized inverse was used rather than a regularized inverse, the cost of encoding the inverse would correspond to the cost of updating the reverse layers by, e.g., a stochastic gradient descent. This update has a cost similar to BP. However, it is unclear whether these updates get us close to the actual inverses.

\paragraph{Bounding the difference between target  propagation and gradient back-propagation}
As the computational graphs of BP and TP are the same, we can  bound the difference between the oracles returned by both methods. First, note that the updates of the parameters of the output functions  are the same since, in TP,  gradients steps of the loss are used to update these parameters. The difference between TP and BP  lies in the updates with respect to the parameters of the transition operations. For BP, the updates are computed by chain rule as 
\[
{\partial_{\param_\hidden} \loss\left(y , \chain_{\theta}( \inpt_{1:\horizon})\right) }
= \sum_{t=1}^{\horizon} 
\partial_{\param_{\hidden}} \dyn_{\param, t} (\hidden_{t-1})
\frac{\partial \hidden_\horizon}{\partial \hidden_t}
{\partial_\hidden \loss(y, \pred_\param(\hidden_\horizon))},
\]
where the term  $\partial \hidden_\horizon /\partial \hidden_t$ decomposes along the time steps as 
${\partial \hidden_\horizon}/{\partial \hidden_t} = \prod_{s=t+1}^{\horizon}  \partial_{\hidden} \dyn_{\param, s}(\hidden_{s-1})$.
The direction computed by TP has the same structure, namely it can be decomposed for $\stepsize_\hidden = 1$  as
\[
\descent_\param = \sum_{t=1}^{\horizon} 
\partial_{\param_{\hidden}} \dyn_{\param, t} (\hidden_{t-1})
\frac{\hat \partial  \hidden_\horizon}{\hat \partial \hidden_t}
{\partial_\hidden \loss(y, \pred_\param(\hidden_\horizon))},
\]
where ${\hat \partial  \hidden_\horizon}/{\hat \partial \hidden_t} =  \prod_{s=t+1}^{\horizon} {\partial_{\hidden}\dyn_{\param, s}^{-1}(\hidden_s) }^\top$. We can then bound the difference between the directions given by BP or TP as, for any matrix norm $\|\cdot\|$ as formally stated in the following lemma.

\begin{restatable}{lemma}{bounddiff}
	The difference between the oracle returned by gradient back-propagation $\partial_{\param_\hidden} \loss\left(y , \chain_{\theta}( \inpt_{1:\horizon})\right)$ and the oracle returned by target propagation can be bounded as 
	\[
	\|\partial_{\param_\hidden} \loss\left(y , \chain_{\theta}( \inpt_{1:\horizon})\right)  - \descent_\param\| 
	\leq c\sup_{t=1, \ldots, \horizon} \| \partial_{\hidden}  \dyn_{\param, t}(\hidden_{t-1}) - {\partial_{\hidden}\dyn_{\param, t}^{-1}(\hidden_t) }^\top \|,
	\]
	where $c =  \sum_{t=1}^\horizon \sum_{s=0}^{t-1} a^sb^{t-1-s}$ with $a = \sup_{t=1, \ldots \horizon} \|  \partial_{\hidden}  \dyn_{\param, t}(\hidden_{t-1})\|, b =  \sup_{t=1, \ldots \horizon} \|  {\partial_{\hidden}\dyn_{\param, t}^{-1}(\hidden_t) }^\top\|$. 
	
	For regularized inverses, we have, denoting $\inter_t = W_{ \inpt\hidden} x_t + W_{\hidden\hidden} \hidden_{t-1} + b_\hidden$, 
	\begin{align*}
		\| \partial_{\hidden}  \dyn_{\param, t}(\hidden_{t-1}) - {\partial_{\hidden}\dyn_{\param, t}^{-1}(\hidden_t) }^\top \|   \leq  
		\|W_{\hidden\hidden}^\top \| \Big( & \|\nabla \activ ( \inter_t) -\nabla \activ(\inter_t)^{-1}\|    + \|\idm - (W_{\hidden\hidden}^\top W_{\hidden\hidden} +\reg \idm)^{-1} \| \|\nabla \activ(\inter_t)^{-1}\| \Big).
	\end{align*}
\end{restatable}
For the two oracles to be close, we then need the preactivation $\inter_t = W_{ \inpt\hidden} x_t + W_{\hidden\hidden} \hidden_{t-1} + b_\hidden$ to lie in the region of the activation function that is close to being linear s.t. $\nabla \activ ( \inter_t) \approx \idm $. We also need $(W_{\hidden\hidden}^\top W_{\hidden\hidden} +\reg \idm)^{-1}$ to be close to the identity which can be the case if, e.g., $\reg=0$ and the weight matrices $W_\hidden$ were orthonormal. By initializing the weight matrices as orthonormal matrices, the differences between the two oracles can be closer. However, in the long term, target propagation appears to give better oracles, as shown in the experiments below. 

\originalparagraph{Target propagation as a Gauss-Newton method?}

Recently target propagation has been interpreted as an approximate Gauss-Newton method, by considering that the difference target propagation formula approximates the linearization of the inverse, which itself is a priori equal to the inverse of the gradients~\citep{bengio2020deriving, meulemans2020theoretical, meulemans2021credit}. 
Namely, provided that $\dyn_{\param, t}^{-1}(\dyn_{\param, t}(\hidden_{t-1})) \approx \hidden_{t-1}$ such that $\partial_\hidden \dyn_{\param, t}(\hidden_{t-1}) \partial_\hidden\dyn_{\param, t}^{-1}(\hidden_t) \approx \idm$, we have
\[
{\partial_{\hidden}\dyn_{\param, t}^{-1}(\hidden_t) }  \approx \left({\partial_\hidden \dyn_{\param, t}(\hidden_{t-1})}\right)^{-1}.
\]
By composing the inverses of the gradients, we get an update similar to the one of Gauss-Newton (GN) method. Namely, recall that if $n$ invertible functions $f_1, \ldots, f_n$ were composed to solve a least square problem of the form $\|f_n \circ \ldots \circ f_1(x) - y\|_2^2$, a Gauss-Newton update would take the form $x^{(k+1)} = x^{(k)} - \partial_{x_0} f_1(x_0)^{-\top} \ldots\partial_{x_{n-1}} f(x_{n-1})^{-\top}(x_n - y)$, where $x_t$ is defined iteratively as $x_0 = x^{(k)}$, $x_{t+1} = f_t(x_t)$. In other words, GN and TP share the idea of composing the inverse of gradients. However, numerous differences remain as detailed in Appendix~\ref{app:proofs}.
Note that even if TP was approximating GN, it is unclear whether GN updates are adapted to stochastic problems. In any case, by using an analytical formula for the inverse, we can test this interpretation by using non-regularized inverses, which would amount to directly use the inverses as in a GN method. If the success of TP could be explained by its interpretation as a GN method, we should observe efficient training curves when no regularization is added.

%% file: sections/04_exp.tex
\begin{figure}
	\begin{center}
		\includegraphics[width=0.9\linewidth]{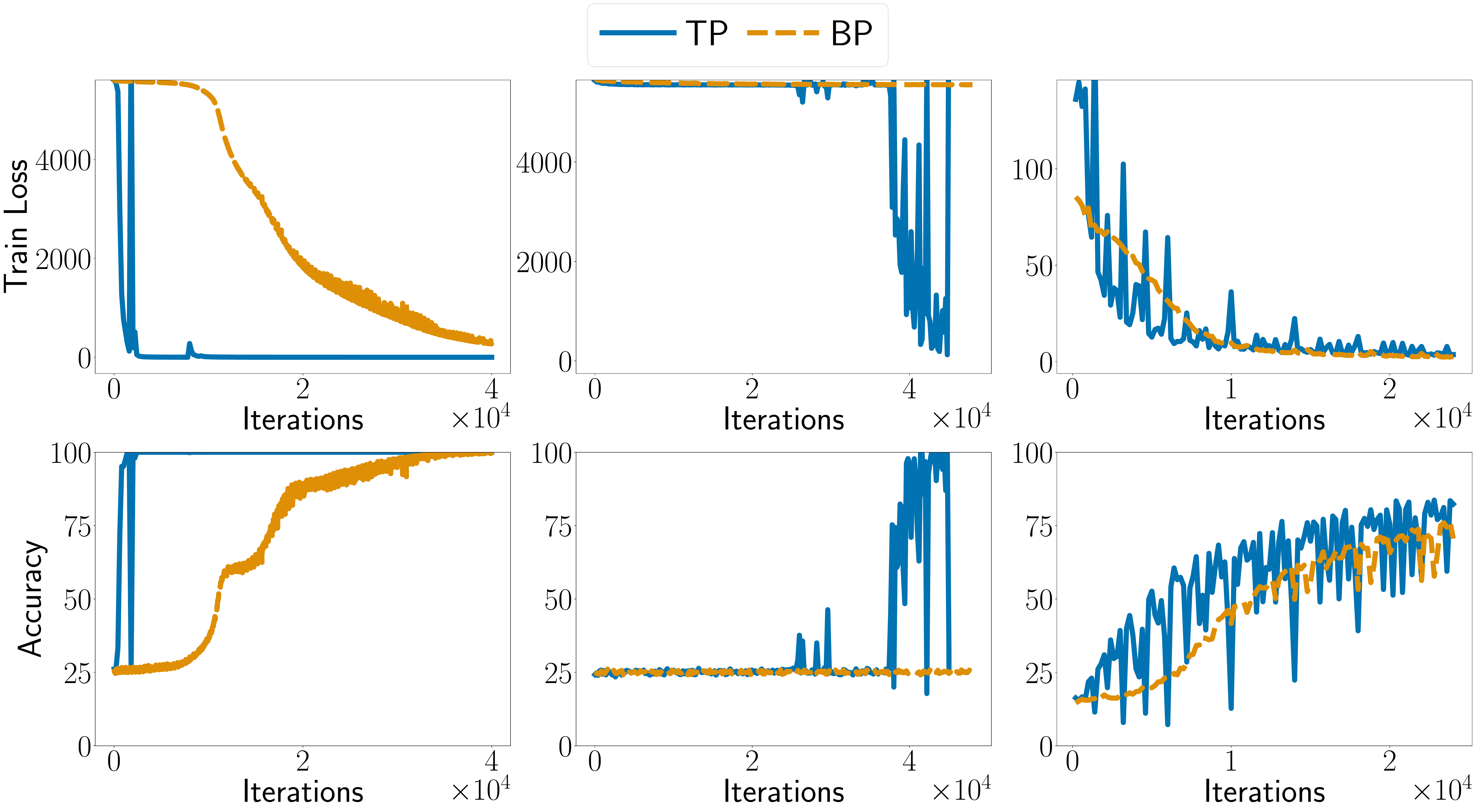}
	\end{center}
	\caption{\small Temporal order problem $T=60$, Temporal Problem $T=120$, Adding problem $T=30$. \label{fig:synth}}
\end{figure}

In the following, we compare  our simple  target propagation approach, which we shall refer to as \textbf{\TPP}, to gradient Back-Propagation referred to as  \textbf{BP}.  We follow the experimental benchmark of~\citet{manchev2020target} to which we add results on RNNs on CIFAR and GRUs on FashionMNIST. Additional experimental,  details on the initialization and the hyper-parameter selection can be found  in Appendix~\ref{app:exp_details}.

\paragraph{Data}
We consider two synthetic datasets generated to present training difficulties for RNNs and several real datasets consisting of scanning images pixel by pixel to classify them~\citep{hochreiter1997long, le2015simple, manchev2020target}.

\textit{Temporal order problem.} A sequence of length $T$ is generated using a set of randomly chosen symbols $\{a, b, c, d\}$. Two additional symbols $X$ and $Y$ are added at positions $t_1 \in [T/10, 2T/10]$ and $t_2 \in [4T/10, 5T/10]$. The network must predict the correct order of appearance of $X$ and $Y$ out of four possible choices $\{XX,XY, YX, YY\}$. 

\textit{Adding problem.} The input consists of two sequences: one is made of randomly chosen numbers from $[0,1]$, and the other one is a binary sequence full of zeros except at positions $t_1\in [1, T/10]$ and $t_2 \in [T/10, T/2]$. The second position acts as a marker for the time steps $t_1$ and $t_2$.  The goal of the network is to output the mean of the two random numbers of the first sequence $(X_{t_1} + X_{t_2})/2$.

\textit{Image classification pixel by pixel.} The inputs are images of (i) grayscale handwritten digits given in the database MNIST~\citep{lecun2010mnist},  (ii) colored objects from the database CIFAR10~\citep{krizhevsky2009learning} or (iii) grayscale images of clothes from the database FashionMNIST~\citep{xiao2017fashion}. The images are scanned pixel by pixel and channel by channel for CIFAR10, and fed to a sequential network such as a simple RNN or a GRU network~\citep{cho2014learning}. The inputs are then sequences of $28\times 28 = 784$ pixels for MNIST or FashionMNIST and $32\times 32\times 3= 3072$ pixels for CIFAR with a very long-range dependency problem. We also consider permuting the images of MNIST by a fixed permutation before feeding them into the network, which gives potentially longer dependencies in the sequential data.
\begin{figure}[t]
	\begin{center}
		\includegraphics[width=0.9\linewidth]{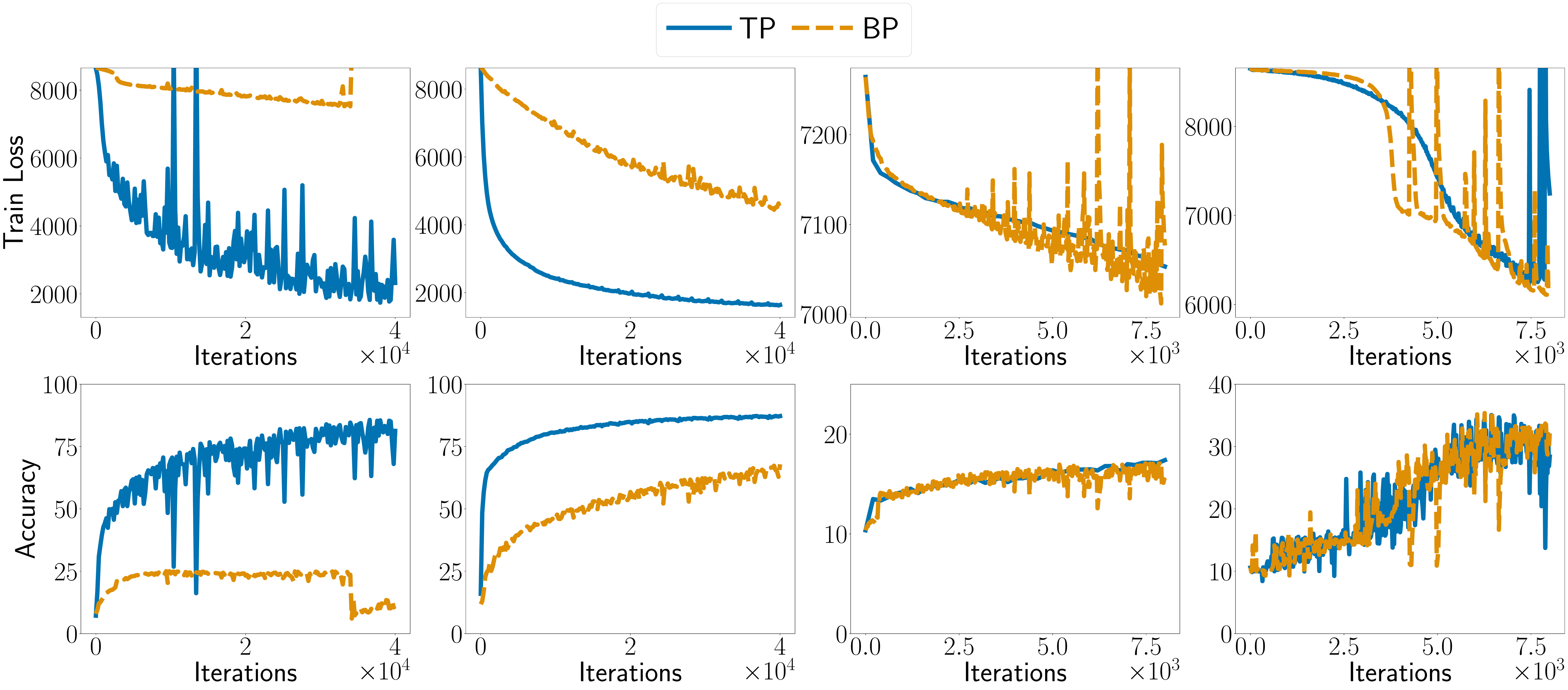}
	\end{center}
	\caption{\small Image classification pixel by pixel. From left to right: MNIST,  MNIST with permuted images, CIFAR10, FashionMNIST with GRU. \label{fig:exp_mnist}}
\end{figure}

\paragraph{Model}
In both synthetic settings, we consider randomly generated mini-batches of size 20, a simple RNN with hidden states of dimension 100, and hyperbolic tangent  activation. For the temporal order problem, the last layer uses a soft-max function on top of a linear operation, and the loss is the cross-entropy.  For the adding problem, the last layer is  linear, the loss is the mean-squared error, and a sample is considered to be accurately  predicted if the mean squared error is less than 0.04 as done by~\citep{manchev2020target}. 

For the classification of images with sequential networks, we consider  mini-batches of size 16 and a cross-entropy loss. For  MNIST and CIFAR, we consider a simple RNN with hidden states of dimension 100, hyperbolic tangent activation, and a softmax output. For FashionMNIST, we consider a GRU network and adapted our implementation of target propagation to that case while using hidden states of dimension 100 and a softmax output.

\paragraph{Target propagation can tackle long sequences better than gradient back-propagation}
In Fig.~\ref{fig:synth}, we observe that \TP  performs better than BP on the temporal ordering problem: it is able to reach 100\% accuracy in fewer iterations than BP for sequences of length 60 and, for sequences of length 120, it is still able to reach 100\% accuracy in fewer than 40 000 iterations while BP is not. On the other hand, for the adding problem, \TP performs less well than BP. The contrast in performance between the two synthetic tasks was also observed by~\citep{manchev2020target} using difference target propagation with parameterized inverses. The main difference between these tasks is the different nature of the outputs, which are binary for the temporal problem and continuous for the adding problem.
\\

In Fig.~\ref{fig:exp_mnist}, we observe that \TP  generally performs better than BP for image classification tasks. For the MNIST dataset, it reaches around 74\% accuracy after $4\cdot 10^4$ iterations. This phenomenon is also observed with permuted images, where the optimization appears smoother, and \TP obtains around 86\% accuracy after $4\cdot 10^4$ iterations and is still faster than  BP. On the CIFAR dataset, no algorithms appear to reach a significant accuracy, though \TP is still faster. On the FashionMNIST dataset, where a GRU network is used, our implementation of \TP performs on par with BP, which shows that our approach can be generalized to more complex networks than a simple RNN.

\paragraph{Target propagation requires a non-zero regularization term}
As mentioned in Sec.~\ref{sec:graph}, by using an analytical formula to compute the inverse of the layers, we can question the interpretation of TP as a Gauss-Newton method, which would amount to TP without regularization.
To understand the effect of the regularization term, we computed the area under the training loss curve of \TP for 400 iterations on a $\log_{\textrm{10}}$ grid of varying step-sizes $\stepsize_\param$ and regularizations $r$ for a fixed $\stepsize_\hidden=10^{-3}$.  The results are presented in Fig.~\ref{fig:heatmap}, where the smaller the area, the brighter the point and the absence of dots in the grid mean that the algorithm diverged. 
Fig.~\ref{fig:heatmap} shows that without regularization we were not able to obtain convergence of the algorithm.
Simply using the gradients of the inverse as in a Gauss-Newton method may not directly work for RNNs. Additional modifications of the method could  be added to make target propagation closer to Gauss-Newton, such as inverting the layers with respect to their parameters as proposed by \citet{bengio2020deriving}. For now, the regularization appears to successfully  handle the rationale of target propagation.

\paragraph{Target propagation is adapted for long sequences}
In Fig.~\ref{fig:regimes}, we compare the performance of BP and \TP in terms of accuracy after 400 iterations on the MNIST problem for various widths determined by the size of the hidden states and various lengths determined by the size of the inputs (i.e., we feed the RNN with $k$ pixels at a time, which gives a length $784/k$). Fig~\ref{fig:regimes} shows that TP is generally appropriate for long sequences, while BP remains more efficient for short sequences. TP can then be seen as an interesting alternative for dynamical problems which involve many discretization steps as in RNNs  and related architectures.

\begin{figure}[t]
	\centering
	\renewcommand{\thefigure}{}
	\hspace{-3em}
	\begin{subfigure}[t]{0.5\textwidth}
		\centering
		\renewcommand{\thesubfigure}{6a}
		\includegraphics[height=0.6\linewidth]{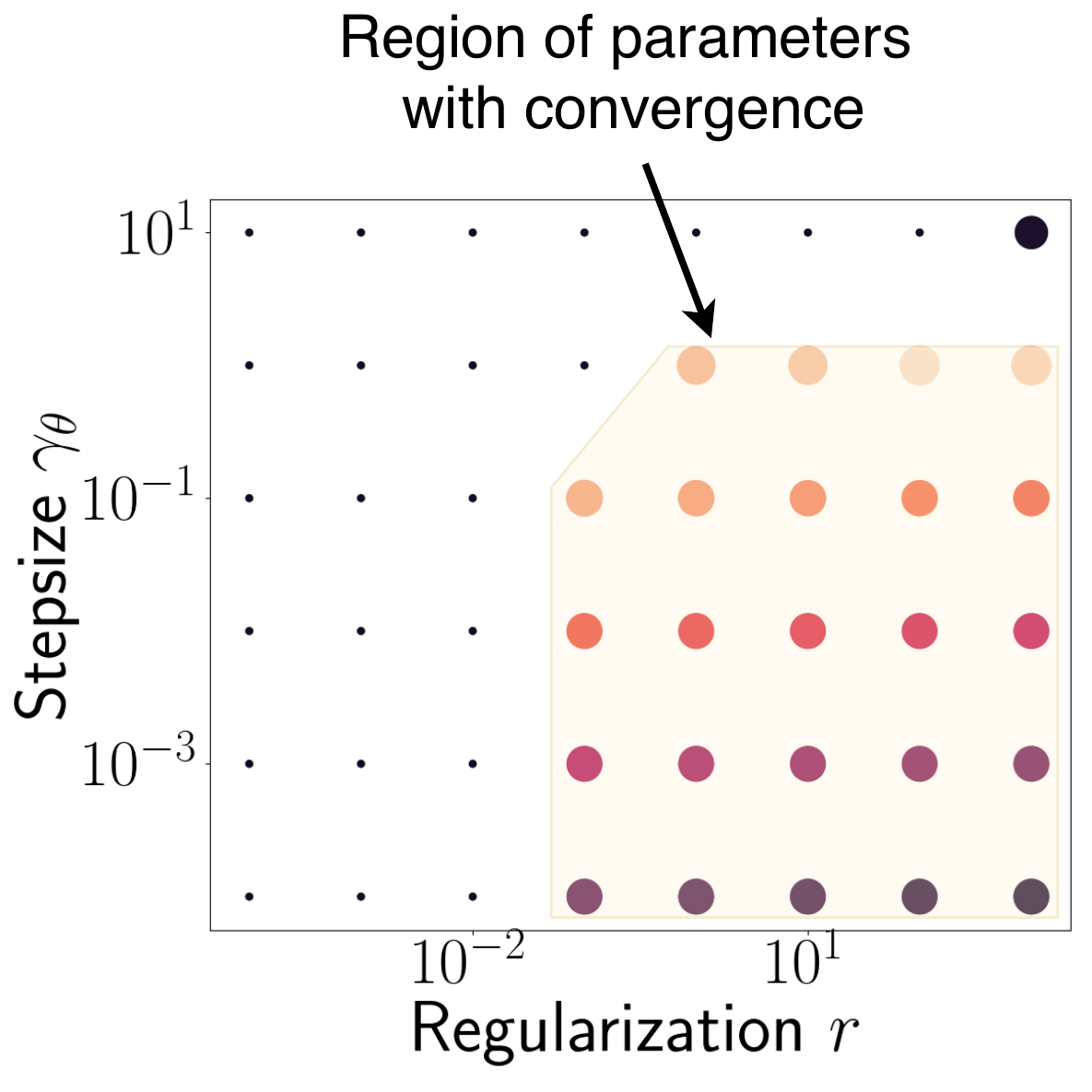}
		\caption{Conv.
			w.r.t. stepsize \& regularization\label{fig:heatmap}}
	\end{subfigure}~
	\begin{subfigure}[t]{0.4\textwidth}
		\centering
		\renewcommand{\thesubfigure}{6b}
		\includegraphics[height=0.65\linewidth]{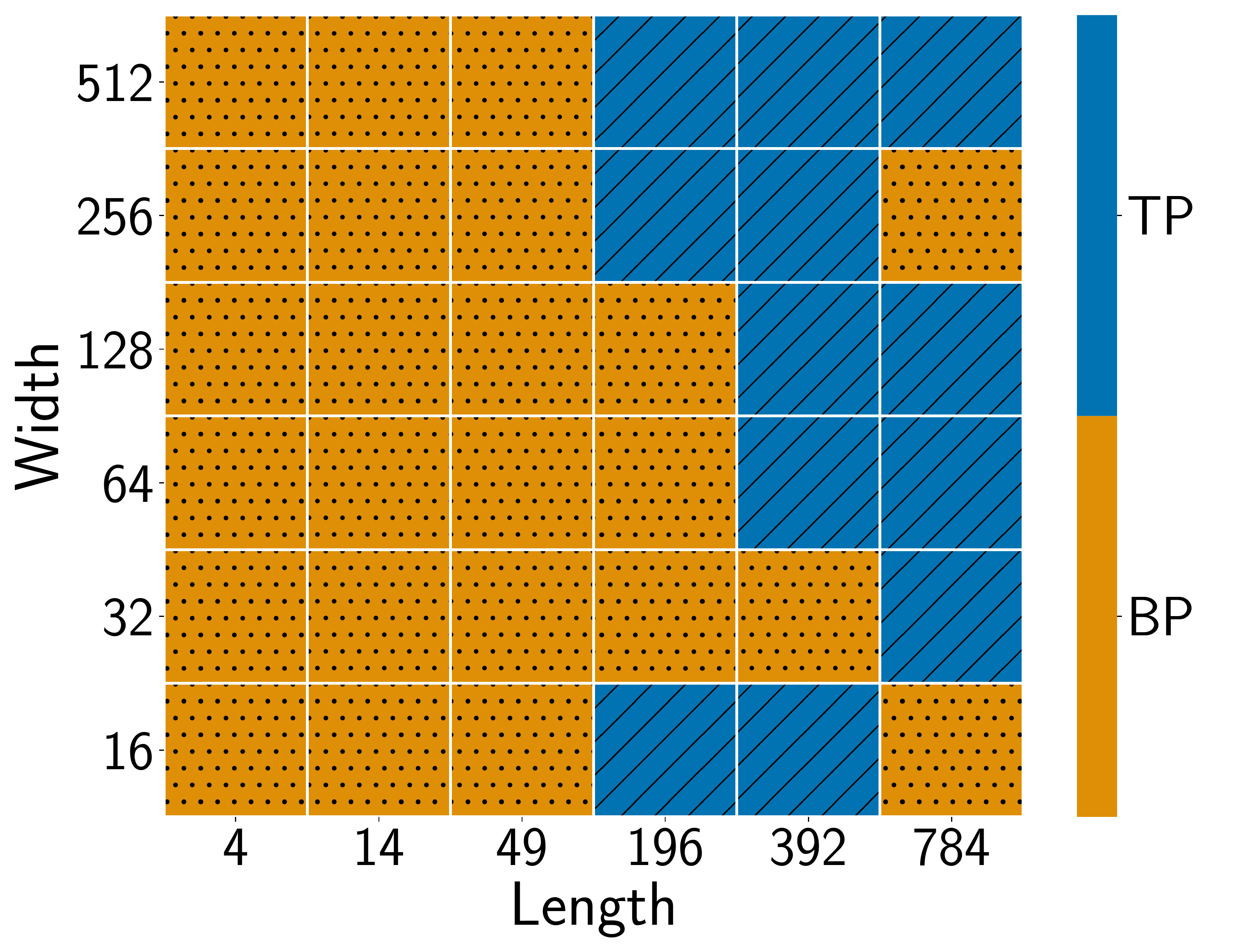}
		\caption{Perf. vs width \& length.\label{fig:regimes}}
	\end{subfigure}
\end{figure}

\section*{Conclusion}
We proposed a simple target propagation approach grounded in two important computational components, regularized inversion, and linearized propagation. The proposed approach also sheds light on previous insights and successful rules for target propagation. The code is available to facilitate the reproduction of the results. We have used target propagation within a stochastic gradient outer loop to train neural networks for a fair comparison to stochastic gradient using gradient backpropagation. Developing adaptive stochastic gradient algorithms in the spirit of Adam that lead to boosts in performance when using target propagation instead of gradient backpropagation is an interesting avenue for future work. 
Continuous counterparts of target propagation in a neural ODE spirit is also an interesting avenue for future work.

%% file: appendix/a_rnn.tex
Given differentiable activation functions $\activ$, the training of recurrent neural networks is amenable to optimization by gradient descent. The gradients can be computed by gradient back-propagation implemented in modern differentiable programming software~\citep{rumelhart1985learning,werbos1994roots, paszke2017automatic}. The gradient back-propagation algorithm is illustrated in Fig.~\ref{fig:grad_prop}. Formally, the gradients are computed by the chain rule such that, for a sample $(y, \inpt_{1:\horizon})$ and $\param_\hidden = (W_{\hidden\hidden}, W_{\inpt\hidden}, b_\hidden)$, 
\[
\frac{\partial \loss\left(y , \chain_{\theta}( \inpt_{1:\horizon})\right) }{\partial \param_\hidden} 
= \sum_{t=1}^{\horizon} 
\frac{\partial \hidden_t}{\partial \param_\hidden}
\frac{\partial \hidden_\horizon}{\partial \hidden_t}
\frac{\partial \hat  \outpt}{\partial \hidden_\horizon}
\frac{\partial \loss }{\partial \hat \outpt}.
\]
The term $\partial \hidden_\horizon /\partial \hidden_t$ decomposes along the time steps as 
\[
{\frac{\partial \hidden_\horizon}{\partial \hidden_t}} = \prod_{s=t+1}^{\horizon} \frac{\partial \hidden_s}{\partial \hidden_{s-1}} .
\] 
As $\horizon$ grows, the norm of the term $\partial \hidden_\horizon /\partial \hidden_t$ can then either increase to infinity (\emph{exploding gradients}) or  exponentially decrease  to  0 (\emph{vanishing gradients}). This phenomenon may prevent the RNN from learning from dependencies between temporally distant events~\citep{hochreiter1998vanishing}. Several solutions were proposed to tackle this issue, including changing the network architecture~\citep{hochreiter1997long}, Hessian-free optimization~\citep{sutskever2011generating}, gradient clipping and regularization~\citep{pascanu2012understanding}, or orthonormal parametrizations~\citep{arjovsky2016unitary, helfrich2018orthogonal, lezcano2019cheap}. We consider here propagating targets instead of gradients as first presented by LeCun and co-workers~\citep{lecun1986learning, lecun1989gemini} and recently revisited by Bengio and co-workers~\citep{bengio2014auto, lee2015difference}.

%% file: appendix/b_target_prop_algo.tex
\subsection{Target Propagation for RNNs}
As detailed in Sec.~\ref{sec:graph},  target propagation with linearized regularized inverses amounts to move along a descent direction computed by a forward-backward algorithm akin to gradient propagation. The iterations of linearized target propagation are then summarized in Algo.~\ref{algo:target_prop_iter}. The iterations of Algo.~\ref{algo:target_prop_iter} make calls to any algorithm providing a descent direction which is  computed by Algo.~\ref{algo:linearized_target_prop}.

In the implementation of the regularized inverses, since the inverse of activation functions such as the sigmoid or the tangent hyperbolic is numerically unstable, we consider projecting on a subset of $a(\reals^\dimhidden)$. For the hyperbolic tangent, we clip the target to $[-1+\varepsilon, 1-\varepsilon]$ for $\varepsilon=10^{-3}$. Concretely, for an hyperbolic tangent activation function, the projection is then 
$
\pi(x) = (\min(\max(x_i, -1+\varepsilon), 1-\varepsilon))_{i=1}^d 
$ for $x \in \reals^d$.
To read Algo.~\ref{algo:linearized_target_prop}, we recall our notations  for $\param =  (W_{\hidden\hidden}, W_{\inpt\hidden}, b_\hidden, W_{\hidden\outpt}, b_\outpt)$:
\begin{align}
	\pred_\param(\hidden_\horizon) & = \alpha(W_{\hidden \outpt} \hidden_\horizon +  b_\outpt), \label{eq:outpt}\\
	\dyn_{\param, t}(\hidden_{t-1}) & = \activ(W_{ \inpt\hidden} x_t + W_{\hidden\hidden} \hidden_{t-1} + b_\hidden), \label{eq:dyn}\\
	\dyn_{\param, t}^{-1}(\target_t) & = (W_{\hidden\hidden}^\top W_{\hidden\hidden} +\reg \idm)^{-1}W_{\hidden\hidden}^\top( \activ^{-1}(\pi(\target_t)) -W_{\inpt\hidden} \inpt_t  - b_\hidden). \label{eq:inv}
\end{align}
Note that Algo.~\ref{algo:linearized_target_prop} can also be used for mini-batches of sequence-output pairs since all operations are either element-wise or linear with respect to the sample of  sequence-output pair. 

\begin{figure}
	\begin{center}
		\includegraphics[width=0.9\linewidth]{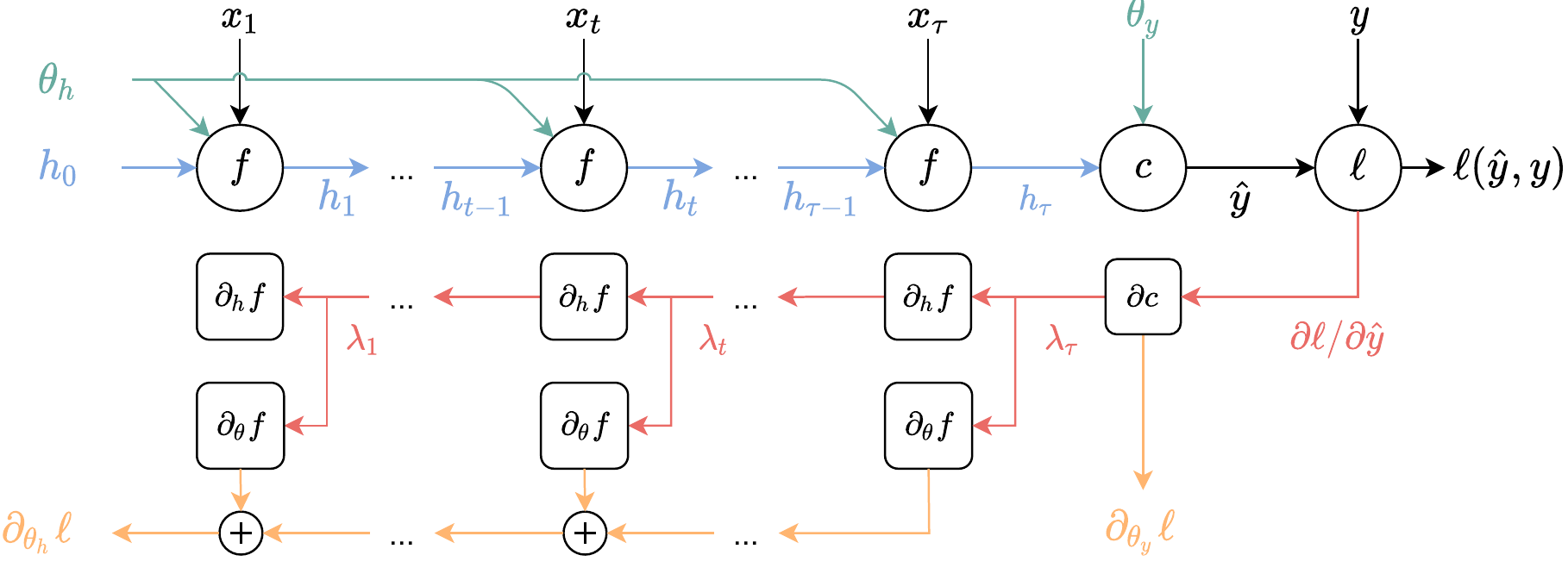}
	\end{center}
	\caption{\small Gradient back-propagation for RNN.\label{fig:grad_prop}}
\end{figure}

\begin{algorithm}[t]\caption{Stochastic learning with target propagation\label{algo:target_prop_iter}}
	\begin{algorithmic}[1]
		\State {\bf Inputs:} Initial  parameters $\param^{(0)} = (W_{\hidden\hidden}, W_{\inpt\hidden}, b_\hidden, W_{\hidden\outpt}, b_\outpt)$ of an RNN  defined by Eq.~\eqref{eq:outpt} and~\eqref{eq:dyn}, stepsize $\stepsize_{\param}$, total number of iterations $K$
		\For{$k=1\ldots K$}
		\State Draw a sample or a mini-batch of  sequences-output pairs
		$(\inpt_{1:\horizon}, y)$.
		\State Compute 
		\[
		\descent_\param = (\descent_{\param_\hidden}, \descent_{\param_\outpt} )= \textrm{TP}(\param^{(k-1)}, \inpt_{1:\horizon}, y),
		\]
		\hspace{13pt}	where $\textrm{TP}$ is Algo.~\ref{algo:linearized_target_prop}
		\State Update the parameters as
		$
		\param^{(k)} = \param^{(k-1)} + \stepsize_{\param} \descent_\param.
		$
		\EndFor
	\end{algorithmic}
\end{algorithm}

\begin{algorithm}\caption{Proposed target propagation algorithm \label{algo:linearized_target_prop}}
	\begin{algorithmic}[1]
		\State{\bf Parameters:}  $\pi$ a projection onto a susbet of $\activ(\reals^\dimhidden)$, stepsize $\stepsize_\hidden$, regularization $\reg$.
		\State {\bf Inputs:} Current parameters $\param= (\param_\hidden, \param_\outpt)$ with $\param_\hidden = (W_{\hidden\hidden}, W_{\inpt\hidden}, b_\hidden), \param_\outpt = (W_{\hidden\outpt}, b_\outpt)$ of the RNN, sample of sequences-output pairs
		$(\inpt_{1:\horizon}, y)$.
		\State	\underline{Forward Pass:}
		\State Compute and store $\darkviolet{V} = (W_{hh}^\top W_{hh} + \reg \idm)^{-1}W_{\hidden\hidden}^\top$  giving access to 
		$
		\dyn_{\param, t}^{-1}(\target_t)   
		$
		defined in Eq.~\eqref{eq:inv}.
		\State Initialize $\mediumblue{\hidden_0}= 0$.
		\For{$t=1,\ldots, \horizon$}
		\State Compute and store
		$
		\mediumblue{\hidden_t} = \dyn_{\param, t}(\mediumblue{\hidden_{t-1}}),\quad  \darkviolet{{\partial_{\param_\hidden} \dyn_{\param, t}(\hidden_{t-1})}}, \quad \violet{\partial_{\hidden_t} \dyn_{\param, t}^{-1}(\hidden_t) }.
		$
		\EndFor
		\State Compute and store
		$
		\loss(\outpt, \pred_\param(\hidden_\horizon)), \quad \darkviolet{ {\partial_{\partial \hidden_\horizon} \loss(y, \pred_{\param}(\hidden_\horizon))}}, \quad \darkviolet{{\partial_{\param_\outpt} \loss(\outpt, \pred_\param(\hidden_\horizon))}}
		$.

		\State	\underline{Backward Pass:}
		
		\State Define 
		$
		\mediumred{\disp_\horizon} = 
		- \stepsize_\hidden 
		\darkviolet{	{\partial_{\hidden_\horizon} \loss(y, \pred_{\param}(\hidden_\horizon))}}, \quad 
		\orange{\descent_{\param_\outpt}} = 
		-\darkviolet{{\partial_{\param_\outpt} \loss(\outpt, \pred_\param(\hidden_\horizon))}}.
		$
		\For{$t= \horizon, \ldots, 1$}
		\State {Compute
			$
			\mediumred{\disp_{t-1}}=   \violet{\partial_{\hidden_t} \dyn_{\param, t}^{-1}(\hidden_t) }^\top \mediumred{\disp_t}.
			$
		}
		\EndFor
		
		\State {\bf Outputs:} Descent directions for $\param_\hidden$, $\param_\outpt$:
		\[
		\orange{\descent_{\param_\hidden}} = \sum_{t=1}^{\horizon}\darkviolet{{ \partial_{\param_\hidden} \dyn_{\param, t}(\hidden_{t-1})} \mediumred{\disp_t}},
		\qquad  \orange{\descent_{\param_\outpt}} = -\darkviolet{{\partial_{\param_\outpt} \loss(\outpt, \pred_\param(\hidden_\horizon))}}.
		\]
	\end{algorithmic}
\end{algorithm}

\subsection{Target-propagation for GRU networks}
\subsubsection{Formulation}
Starting from $h_0 = 0$, given an input sequence $x_1,\ldots, x_\tau$, the GRU network (as implemented in Pytorch\footnote{Compared to {\small \url{ https://pytorch.org/docs/stable/generated/torch.nn.GRU.html}}, we used a single variable $b_m=b_{im}+ b_{hm}$, same for $b_z$.}~\citep{paszke2017automatic}), iterates for $t=1, \ldots, \tau$, 
\begin{align}
	m_t &=  f_{m,t}(h_{t-1}) := \sigma(W_{im} x_t + W_{hm} h_{t-1} + b_m) \\
	z_t & = f_{z, t}(h_{t-1}) := \sigma(W_{iz} x_t + W_{hz} h_{t-1} + b_z) \\
	n_t &= f_{n, t}(h_{t-1},m_t) := \tanh(W_{in} x_t +  b_{in} + m_t \odot  (W_{hn} h_{t-1} + b_{hn})) \\
	h_t & = f_{h, t }(h_{t-1}, z_t, n_t):=  (1 - z_t) \odot h_{t-1} + z_t \odot n_t,
\end{align}
where $\odot$ is the Hadamard product, $\sigma$ is a sigmoid. 
In the following, we will denote simply $\theta =(\theta_m, \theta_z, \theta_n)$ the parameters of the network with
\begin{align*}
	\theta_m = (W_{im}, W_{hm}, b_m), \qquad
	\theta_z = (W_{iz}, W_{hz}, b_z),\qquad
	\theta_n  = (W_{in} , b_{in}. W_{hn}, b_{hn}) .
\end{align*}
The output of the network is e.g. a soft-max operation on the hidden state computed at the last step (if applied to an image scanned pixel by pixel for example). See the main paper for the expression of the output in that case.

\subsubsection{Modifying the chain rule}
The underlying idea of our implementation of target propagation in a differentiable programming framework is to mix classical gradients and Jacobians of the inverse of the functions. Denote  for a given output loss $\mathcal{L}$ computed on a given mini-batch with the current parameters $\theta$, 
$
{\hat \partial \mathcal{L}}/{\hat \partial h_t}
$ the direction back-propagated by our implementation of target propagation until the step $h_t$. The directions for the parameters of the network can be output as 
\[
\frac{\hat \partial \mathcal{L}}{\hat  \partial \theta} = \sum_{t=1}^{\tau}\frac{\partial h_t}{\partial  \theta} \frac{\hat \partial \mathcal{L}}{\hat \partial h_t} .
\]

The main task is to define $
{\hat \partial \mathcal{L}}/{\hat \partial h_{t-1}}
$ given $
{\hat \partial \mathcal{L}}/{\hat \partial h_t}
$ and appropriate regularized inverses. For that, we start with the chain rule for  $
{ \partial h_t}/{ \partial h_{t-1}}
$ and we will replace some of the gradients by Jacobians of regularized inverses at some places. 

\paragraph{Classical chain rule}
We have
\begin{align}\label{eq:partial_hidden_1}
	\frac{\partial h_t}{\partial h_{t-1}} & = \left(- \frac{\partial z_t}{\partial h_{t-1}}\right) \diag(h_{t-1})  + \idm  \diag(1-z_t) +\frac{\partial z_t}{\partial h_{t-1}}  \diag (n_t)  + \frac{\partial n_t}{\partial h_{t-1}} \diag(z_t)  \\
	&  = \diag(1-z_t) +  \frac{\partial z_t}{\partial h_{t-1}}\left(\diag(n_t)-\diag(h_{t-1})\right)  +  \frac{\partial n_t}{\partial h_{t-1}} \diag(z_t).
\end{align}
Now for ${\partial n_t}/{\partial h_{t-1}}$, we further decompose the function $f_{n, t}(h_{t-1})$ as 
\[
f_{n, t}(h_{t-1}) = g_t(m_t \odot a_t ),
\]
with $g_t(u) = \tanh(W_{in}x_t + b_{in} + u)$ and $a_t  =  \ell(h_{t-1}):= W_{hn} h_{t-1} + b_{hn}$. We then have, denoting $u = m_t \odot a_t $ 
\begin{align}\label{eq:partial_gidden_2}
	\frac{\partial n_t}{\partial h_{t-1}} = \left(\frac{\partial m_t}{\partial h_{t-1}} \diag (a_t) + \frac{\partial a_t }{\partial h_{t-1}} \diag(m_t)\right) \nabla g_t(u),
\end{align}
with $\nabla g_t(u) = \diag(\tanh'(W_{in}x_t + b_{in} + u))$.

\paragraph{Inverses}
Now, the variables $z_t, m_t$ and $a_t$ are functions of $h_t$ that incorporate a linear operation and that can be inverted.  Namely, we can define the following regularized inverses
\begin{align*}
	f^{-1}_{m, t}(v_t) & = (W_{hm}^\top W_{hm} + r \idm)^{-1} W_{hm}^\top(\sigma^{-1}(v_t) - W_{ir} x_t - b_m) \\
	f^{-1}_{z, t}(v_t) & = (W_{hz}^\top W_{hz} + r \idm)^{-1} W_{hz}^\top(\sigma^{-1}(v_t) - W_{iz} x_t - b_z) \\
	\ell^{-1}(v_t) & = (W_{hn}^\top W_{hn} + r\idm)^{-1} W_{hn}^\top(v_t - b_{hn}).
\end{align*}
We can then do the following substitutions in Eq.\eqref{eq:partial_hidden_1} and \eqref{eq:partial_gidden_2}
\begin{align*}
	\frac{\partial m_t}{\partial h_{t-1}} & \leftarrow  \frac{\hat \partial m_t}{\hat \partial h_{t-1}}= \nabla f^{-1}_{m, t}(m_t)^\top \\
	\frac{\partial z_t}{\partial h_{t-1}} & \leftarrow  \frac{\hat \partial z_t}{\hat \partial h_{t-1}}= \nabla f^{-1}_{z, t}(z_t)^\top\\
	\frac{\partial a_t}{\partial h_{t-1}} & \leftarrow  \frac{\hat \partial a_t}{\hat \partial h_{t-1}}= \nabla \ell^{-1}(a_t)^\top
\end{align*}
to define the quantity back-propagated by target propagation.

Note that by taking the gradient of the inverse we can ignore the biases and the inputs. Namely, we have for example
\[
\nabla f^{-1}_{m, t}(m_t) = \diag((\sigma^{-1})'(m_t)) W_{hm}(W_{hm}^\top W_{hm} + r \idm)^{-1},
\]
hence 
\[
\nabla f^{-1}_{m, t}(m_t)^{\top} = (W_{hm}^\top W_{hm} + r \idm)^{-1}W_{hm}^\top   \diag((\sigma^{-1})'(m_t))
\]
The expression for $\nabla f^{-1}_{z, t}(z_t)$ is identical. Since $\ell$ is affine, we have simply
\[
\nabla \ell^{-1}(a_t)^\top = (W_{hn}^\top W_{hn} + r\idm)^{-1} W_{hn}^\top.s
\] 

\paragraph{Summary}
Combined together, we get, denoting $d_t = \frac{\hat \partial \mathcal{L}}{\hat \partial h_{t}}$, 
\begin{align*}
	\frac{\hat \partial \mathcal{L}}{\hat \partial h_{t-1}} & 
	= (1-z_t)\odot d_t 
	+  \nabla f^{-1}_{z, t}(z_t)^\top 
	((n_t - h_{t-1})\odot d_t) \\
	& \quad + \nabla f^{-1}_{m, t}(m_t)^\top 
	( a_t \odot \tanh'(W_{in}x_t + b_{in} + u) \odot z_t \odot d_t)\\
	&\quad + \nabla \ell^{-1}(a_t)^\top 
	(m_t \odot \tanh'(W_{in}x_t + b_{in} + u) \odot z_t \odot d_t)\\
	& = (1-z_t)\odot d_t 
	+  (W_{hz}^\top W_{hz} + r \idm)^{-1}W_{hz}^\top 
	\left((\sigma^{-1})'(z_t) \odot (n_t - h_{t-1})\odot d_t\right) \\
	& \quad  + (W_{hm}^\top W_{hm} + r \idm)^{-1}W_{hm}^\top
	\left(( \sigma^{-1})'(m_t) \odot a_t \odot \tanh'(W_{in}x_t + b_{in} + u) \odot z_t \odot d_t\right) \\
	&\quad +(W_{hn}^\top W_{hn} + r\idm)^{-1} W_{hn}^\top (m_t \odot \tanh'(W_{in}x_t + b_{in} + u) \odot z_t \odot d_t).
\end{align*}
This provides a rule to propagate targets through linearized regularized inverses.

%% file: appendix/c_proofs.tex
\subsection{Gradient Back-Propagation vs Target Propagation}
\bounddiff*
\begin{proof}
	The first claim is a direct application of Lemma~\ref{lem:diff_prod_matrix} and the second claim follows from the formulation of the regularized inverse, using that $\nabla \activ^{-1}(\hidden_t) = \nabla \activ( \activ^{-1}(\hidden_t))^{-1} =  \nabla \activ( \inter_t)^{-1} $.
\end{proof}

\begin{lemma}\label{lem:diff_prod_matrix}
	Given $A_1, \ldots, A_n, B_1, \ldots, B_n \in \reals^{n\times n}$, for any matrix norm $\|\cdot\|$, and any $1 \leq t \leq n$,
	\[
	\left\| \prod_{i=1}^t A_i - \prod_{i=1}^t B_i\right\| \leq \delta \sum_{i=0}^{t-1} a^i b^{t-1-i}
	\]
	where $a = \sup_{i=1, \ldots, n} \|A_i\|$, $b = \sup_{i=1, \ldots n} \|B_i\|$ and $\delta = \sup_{i=1, \ldots, n} \|A_i-B_i\|$.
\end{lemma}
\begin{proof}
	Define for $t\geq 1$, $\delta_t = 	\| \prod_{i=1}^t A_i - \prod_{i=1}^t B_i\|$, we have
	\begin{align*}
	\delta_t  & \leq \left\| A_t\left( \prod_{i=1}^{t-1}A_i - \prod_{i=1}^{t-1}B_i\right) + (A_t-B_t) \prod_{i=1}^{t-1} B_i\right\|
	\leq a \delta_{t-1} + \delta b^{t-1} \leq \delta \sum_{i=0}^{t-1} a^i b^{t-1-i}.
	\end{align*}
\end{proof}
A convergence to a stationary point  for TP can be derived from classical results on an approximate gradient descent detailed below (the proof is akin to the results of \citet{devolder2014first}).
\begin{corollary}[Corollary of Lemma~\ref{prop:approx_grad_cvg}]
	Denote $\varepsilon_k$ a bound on the difference between the oracle returned by gradient back-propagation and by target-propagation both applied to the whole dataset. Provided that the objective is $L$-smooth and the stepsizes of TP are chosen such that $\stepsize = \stepsize_\hidden \stepsize_\outpt <1/2L$, after $k$ iterations, we get 
		\[
	\min_{i\in\{0, \ldots, k-1\}}\|\nabla F(\theta_i)\|_2^2 \leq c_1 {\frac{F(\theta_0)  -\min_{\theta \in \reals^d} F(\theta)}{\stepsize k}}  +  \frac{c_2}{{k}}\sum_{i=0}^{k-1} \varepsilon_i^2.
	\]
	where $F(\theta_i) = \frac{1}{n}\sum_{i=1}^{n} \loss(\phi(x_i, \theta), y_i)$ with  $\phi(x_i, \theta)$ the output of the RNN on a sample $x_i$ and $\loss$ the chosen loss. 
\end{corollary}

\begin{lemma}\label{prop:approx_grad_cvg}
	Let $f: \reals^d \rightarrow \reals$ be a $L$-smooth function. Consider an $\varepsilon$-approximate gradient descent on $f$ with step-size $\stepsize\leq1/(2L)$, i.e.,
	$
	\var_{k+1} = \var_k - \stepsize \widehat \nabla f(\var_k),
	$
	where $\|\widehat \nabla f(\var_k) - \nabla f(\var_k)\|_2 \leq \varepsilon_k$. After $k$ iterations, this method satisfies, 	for $c_1, c_2$ two universal constants,
	\[
	\min_{i\in\{0, \ldots, k-1\}}\|\nabla f(\var_i)\|_2^2 \leq c_1 {\frac{f(\var_0)  -\min_{\var \in \reals^d} f(\var)}{\stepsize k}}  +  \frac{c_2}{{k}}\sum_{i=0}^{k-1} \varepsilon_i^2.
	\]
\end{lemma}
\begin{proof}
	Denote $g_k = \widehat \nabla f(\var_k) - \nabla f(\var_k)$ for all $k\geq 0$. By $L$-smoothness of the objective, the iterations of the approximate gradient descent satisfy
	\begin{align*}
		f(\var_{k+1}) & \leq f(\var_k) + \nabla f(\var_k)^\top (\var_{k+1} - \var_k) + \frac{L}{2}\|\var_{k+1}-\var_k\|_2^2 \\
		& =  f(\var_k) - \stepsize \|\nabla f(\var_k)\|_2^2 -\stepsize \nabla f(\var_k)^\top g_k + \frac{L\stepsize^2}{2}\|\nabla f(\var_k)+ g_k\|_2^2 \\
		&  = f(\var_k) - \stepsize\left(1-\frac{L\stepsize}{2}\right)\|\nabla f(\var_k)\|_2^2 + \frac{L\stepsize^2}{2}\|g_k\|_2^2  + \stepsize(L\stepsize - 1)\nabla f(\var_k)^\top g_k \\
		& \leq f(\var_k) - \stepsize\left(1-\frac{L\stepsize}{2}\right)\|\nabla f(\var_k)\|_2^2 + \frac{L\stepsize^2}{2}\|g_k\|_2^2  + \stepsize(1- L\stepsize)\|\nabla f(\var_k)\|_2 \|g_k\|_2,
	\end{align*}
	where in the last inequality we bounded the absolute value of the last term and used that $\stepsize L \leq 1$. Now we use that for any $a,b \in \reals$ and $\theta>0$, $2ab \leq \theta a^2 + \theta^{-1} b^2$, which gives for $\theta >0$, $a = \sqrt{\stepsize(1- L\stepsize)/2}\|\nabla f(\var_k)\|_2$ and $b = \sqrt{\stepsize(1- L\stepsize)/2} \|g_k\|_2$, 
	\begin{align*}
		f(\var_{k+1}) \leq f(\var_k) - \stepsize\left(1-\frac{L\stepsize+\theta(1-L\stepsize)}{2} \right)\|\nabla f(\var_k)\|_2^2 + \frac{L\stepsize^2 + \theta^{-1}\stepsize(1-L\stepsize)}{2}\|g_k\|_2^2.
	\end{align*}
	Using $ 0\leq L\stepsize \leq 1/2$, $\theta =1/4$ and $\|g_k\|_2^2 \leq \varepsilon_k^2 $, we get 
	$
	f(\var_{k+1}) \leq f(\var_k) - \frac{11}{16}\stepsize \|\nabla f(\var_k)\|_2^2 + 2\stepsize\varepsilon_k^2.
	$
	Rearranging the terms, summing from $i=0, \ldots, k-1$, taking the minimum, dividing by $k$ we get the result.  
\end{proof}

\subsection{Target Propagation vs Gauss-Newton updates}
We discuss the interpretation of Target Propagation (TP) as a Gauss-Newton (GN) method which was proposed by
\citet{bengio2020deriving, meulemans2020theoretical}.
As already mentioned in Sec.~\ref{sec:graph}, the main similarity between TP and GN is the fact that both TP and GN use the inverse or approximations of inverses of the gradients. In this section, we shall discuss this interpretation for feed-forward networks to follow the claims of \citet{meulemans2020theoretical}. Namely, we consider here a network defined by $L$ weights $W_1, \ldots, W_L$ and $L$ activation functions $a_1, \ldots, a_L$ which transform an input $x_0$ into an output $x_L$ by computing (no biases were considered by~\citet{meulemans2020theoretical}),
\[
x_{t} = f_{t}(x_{t-1}) = a_t(W_t x_{t-1}) \quad \mbox{for} \ t \in \{1, \ldots, L\}
\]
Denoting $\phi(x;\theta)$ the output of the network for an input $x= x_0$, with $\theta = (W_1, \ldots, W_L)$ being the parameters of the network, the objective consists in minimizing the loss between the outputs of the network and the sample outputs, i.e., minimizing $\mathcal{L}(y, \phi(x;\theta))$ for pairs of inputs outputs samples $(x, y)$. 

\paragraph{GN step}
Recall first the rationale of a GN step for such feed-forward networks with a squared-loss, which amount to solving 
\[
\min_{\theta\in \reals^p} \frac{1}{n} \sum_{i=1}^n \|\phi(x_i;\theta) - y_i\|_2^2 ,
\]
with $y_i \in \reals^{K}$ (for classification in $K$ classes) and $\phi(x_i, \theta) \in \reals^{d_L}$. 
A GN step amounts to linearize the non-linear function $\phi$ around a current set of parameters $\theta^{(k)}$ and solve the corresponding least-square problems to define the next set of parameters, i.e,
\begin{align*}
\theta^{(k+1)} & = \argmin_\theta \frac{1}{n} \sum_{i=1}^n \|\phi(x_i;\theta^{(k)}) + \partial_\theta \phi(x_i;\theta^{(k)})^\top(\theta- \theta^{(k)}) - y_i\|_2^2 \\
& = \theta^{(k)} - 
\left(\sum_{i=1}^n \partial_\theta \phi(x_i;\theta^{(k)})\partial_\theta \phi(x_i;\theta^{(k)})^\top \right)^{-1}
\left( \sum_{i=1}^n \partial_\theta \phi(x_i;\theta^{(k)})  \left(\phi(x_i, \theta^{(k)}) - y_i\right)\right).
\end{align*}
To consider TP as an approximate GN method we need the following considerations. 
\begin{enumerate}
	\item Consider the iteration on a mini-batch of size 1, s.t.
	\[
	\theta^{(k+1)} = \theta^{(k)} - 
	\left(\partial_\theta \phi(x_i;\theta^{(k)})\partial_\theta \phi(x_i;\theta^{(k)})^\top \right)^{-1}
	\left(  \partial_\theta \phi(x_i;\theta^{(k)})  \left(\phi(x_i, \theta^{(k)}) - y_i\right)\right).
	\]
	\item Consider that the gradients of the networks are invertible, s.t. 
	\[
	\theta^{(k+1)} = \theta^{(k)} - 
	\left(\partial_\theta \phi(x_i;\theta^{(k)})\right)^{-\top}
	  \left(\phi(x_i, \theta^{(k)}) - y_i\right).
	\]
	\item Consider updating only one set of parameters $\theta_l = W_l$ , s.t., 
		\[
	\theta_l^{(k+1)} = \theta_l^{(k)} - 
	\left(\partial_{\theta_l} \phi(x_i;\theta^{(k)})\right)^{-\top}
	\left(\phi(x_i, \theta^{(k)}) - y_i\right).
	\]
	with 
	\[
	\partial_{\theta_l} \phi(x_i;\theta^{(k)}) = \partial_{\theta_l} f_l(x_{l-1})   \partial_{x} f_{l+1}(x_l)  \ldots \partial_{x} f_{L}(x_{L-1})
	\]
	so that, provided that all matrices inside the matrix multiplication are invertible,  we get
	\[
	\partial_{\theta_l} \phi(x_i;\theta^{(k)})^{-T} = \partial_{\theta_l} f_l(x_{l-1})^{-T}   \partial_{x} f_{l+1}(x_l)^{-T}  \ldots \partial_{x} f_{L}(x_{L-1})^{-T}
	\]
	\item Finally, ignore the last inversion and replace it by a gradient step on the parameters $\theta_l$, then we get an iteration similar to TP, with 
			\[
	\theta_l^{(k+1)} = \theta_l^{(k)} - 
\partial_{\theta_l} f_l(x_{l-1})   \partial_{x} f_{l+1}(x_l)^{-T}  \ldots \partial_{x} f_{L}(x_{L-1})^{-T} \partial_{x_L} \mathcal{L}(y, x_L)
	\]
	for $\mathcal{L}$ a squared loss. Namely, we keep the inversion of the gradients of the intermediate functions. 
\end{enumerate}
Our objective here is to question whether viewing TP as a GN step with the approximations explained above is meaningful or not.

\originalparagraph{Does the original TP formulation approximate GN?}
\citet{meulemans2020theoretical} start by considering the original TP formulation, i.e., targets computed as 
$
\target_t = \psi_t(\target_{t+1})
$
for $\psi_t$ an approximate inverse of $f_t$ and with $\target_L = x_L - \eta \partial_{x_L}\loss(y, x_L)$. \citet[Lemma 1]{meulemans2020theoretical} show then that, provided that we use the exact inverse, $\psi_t = f_t^{-1}$, 
\[
\Delta x_t = \target_t - x_t = - \eta \prod_{s=t}^{L-1} \partial_{x_s} f_{s+1}(x_s)^{-\top}  \partial_{x_L}\loss(y, x_L) + O(\eta^2).
\]
\citep[Theorem 2]{meulemans2020theoretical} conclude that (i) for mini-batches of size 1, (ii) for a squared loss, (iii) for invertible $f_t$, as $\eta \rightarrow 0$, TP uses a Gauss-Newton optimization with block diagonal approximation to compute the targets  in the sense that as $\eta \rightarrow 0$, 
\[
\Delta x_t  \approx - \eta \partial_{x_t} (f_{t+1} \circ  \ldots \circ f_L)^{-\top }(x_t).
\]
As the stepsize of any optimization algorithm tends to 0, they all are the same, since the update would be 0 in all cases. An optimization algorithm aims not to have infinitesimal stepsizes.
To make the claim of \citet{meulemans2020theoretical} more precise, the constants hidden in $O(\eta^2)$ need to be detailed in order to understand in which regimes of the stepsize the approximation is meaningful. Assuming the inverses $\psi_t$ to be $\ell$ Lipschitz continuous and $L$-smooth (i.e. with $L$-Lipschitz continuous gradients), a quick look at the proof of Lemma 1 of \citet{meulemans2020theoretical} shows that 
\begin{align*}
\target_t - x_t & = - \eta \prod_{s=t}^{L-1} \partial_{x_s} f_{s+1}(x_s)^{-\top}  \partial_{x_L}\loss(y, x_L) + \xi_t\\
\|\xi_t\|_2 &\leq a_t \\
a_s & \leq La_{s+1}^2 + \ell a_{s+1} + L\ell^2 \eta^2\| \partial_{x_L}\loss(y, x_L)\|_2^2 \quad \mbox{for} \ s \in\{t, \ldots, L-1\} \\
a_L & \leq \frac{L}{2}\eta^2\| \partial_{x_L}\loss(y, x_L)\|_2^2.
\end{align*}
The above bound shows that $\|\xi_t\|_2$ grows w.r.t. the stepsize $\eta$ as a polynomial with leading term  $\eta^{2^{L-t}}$. So unless the stepsize is extremely small, it seems unclear whether the original TP  formulation approximates GN in this case. Though the above bound may be pessimistic, it captures correctly the dependency of the error w.r.t. $\eta$. True, as $\eta \rightarrow 0$, the leading term is in $O(\eta^2)$, but generally stepsizes are computed for $\eta$ not infinitesimally small.

Finally, if the similarity of TP with GN could explain its efficiency, then by the reasoning of~\citet{meulemans2020theoretical}, the original TP formulation should be efficient. Yet, the original TP formulation has never been shown to produce satisfying results.

\originalparagraph{Does TP with the difference target propagation approximate GN?}
\citet{meulemans2020theoretical} make a similar claim for TP with the Difference Target Propagation formula, i.e., $\target_t = x_t + \psi_t(\target_{t+1}) - \psi_t(x_{t+1})$. Namely, \citet[Lemma 3]{meulemans2020theoretical} show that 
\[
\Delta x_t = \target_t - x_t = - \eta \prod_{s=t}^{L-1} \partial_{x_s} \psi_{s}(x_s)^{\top}  \partial_{x_L}\loss(y, x_L) + O(\eta^2).
\]
Once again, for the claim to be meaningful beyond infinitesimal stepsizes, the terms in $O(\eta^2)$ need to be detailed. A quick look at the proof of \citet[Lemma 3]{meulemans2020theoretical} shows that under appropriate smoothness assumptions the error can be bounded as a polynomial in $\eta$ with a leading term  $\eta^{2^{L-t}}$. So again, unless we consider infinitesimal stepsizes, it is unclear whether this approximation is useful.  

\paragraph{Linearized target propagation and GN}
If we use a linearized version of the difference target propagation formula as presented in~\eqref{eq:newtarg}, namely $\target_t -x_t = \partial_{x_{t+1}} \psi_t(x_{t+1})^{\top}(\target_{t+1} - x_{t+1})$ , then we have the \emph{equality} 
\[
\Delta x_t = \target_t - x_t = - \eta \prod_{s=t}^{L-1} \partial_{x_s} \psi_{s}(x_s)^{\top}  \partial_{x_L}\loss(y, x_L) 
\]
and the idea that TP could be seen as an approximate GN method may be pursued in a meaningful way. However the error of approximation of the inverse of the gradients must be taken into account in order to understand the validity of the approach. 

\paragraph{Propagating the approximation error of the gradient inverses}
We  compute the approximation error incurred by composing gradients of the inverse instead of inverses of gradients.
Formally, the  approximation error for one layer can be estimated  under the assumption that 
\begin{equation}\label{eq:asm_approx}
	\psi_{t}(\dyn_{t}(x_{t-1})) = x_{t-1} + e(x_{t-1}),
\end{equation}
 with $e$ an  $\varepsilon$-Lipschitz continuous  function and the assumption that  the minimal singular value $ \sigma$  of ${{\partial_{x_t} f_t(x_{t-1})}}$ is positive. 

The function $e$ a priori depends on $\param$; we ignore this dependency and simply consider $e$ to be $\varepsilon$-Lipschitz continuous for all $\param$. For a  function $e$, we define its Lipschitz continuity constant as 
$
\varepsilon = \sup_{x} \sup_{\|\lambda\|_2\leq 1}  \|\partial_x e(x)^\top \lambda\|_2 = \sup_x \|\partial_x e(x)\|,
$
where $\|\cdot\|$ denotes the spectral norm.
By differentiating both sides of Eq.~\eqref{eq:asm_approx}, we get
\[
\partial_{x} \dyn_{ t}(x_{t-1}) \partial_{x} \psi_t(x_t) = \idm + \partial_x e(x_{t-1}).
\]
By assuming the minimal singular value $\sigma$ of $\partial_{x} \dyn_{t}(x_{t-1})$ to be positive, we get that $\partial_{x} \dyn_{t}(x_{t-1}) $ is invertible and so 
\[
\partial_{x} \psi_t(x_t) = \partial_{x} \dyn_{t}(x_{t-1})^{-1} (\idm + \partial_x e(x_{t-1})).
\]
Hence 
\begin{equation}\label{eq:approx_layer}
	\|\left(\partial_x \dyn_{t}(x_{t-1})\right)^{-1}- \partial_x \psi_t(x_t)\|
	\leq
	\frac{\varepsilon}{\sigma},
\end{equation}
and $\partial_x \psi_t(x_t)$ is $\sigma^{-1}(1+\varepsilon)$ Lipschitz-continuous. 

Now for multiple compositions, using Lemma~\ref{lem:diff_prod_matrix}, we get 
\[
\|\left(\partial_\hidden \dyn_{1}(x_{0})\right)^{-1}\ldots \left(\partial_x \dyn_L(x_{L-1})\right)^{-1} 
- \partial_x \psi_{1}(x_1) \ldots \partial_x \psi_L(x_L)\| \leq \frac{(1+ \varepsilon)^L}{\sigma^L}.
\]
Therefore the accumulation error diverges with  the length $L$ of the network as soon as $\varepsilon \geq \sigma -1$.

\paragraph{Testing the hypothesis that TP could be interpreted as using GN updates directions}
Here we come back to the setting of RNNs presented in the paper. In this case the length of the compositions of layers is $\horizon$ and according to the previous discussion, the error of approximation of the product of the inverse of the gradients by the product of the gradients of the approximate inverses could easily diverge as $\horizon$ grows (long sequences). Nevertheless, by using analytical formulas for the inverses, we can ensure that the approximation error is zero, which would correspond then to the ideal setting where TP uses GN update directions for the hidden states.

Formally, in the context of RNNs, a Gauss-Newton update direction for the hidden states is given as (ignoring the inverse of the output function) 
\[
-\stepsize_\hidden \prod_{s=t+1}^{\horizon-1} \left(\partial_{\hidden} \dyn_{t+1, \param}(\hidden_t)\right)^{-\top}
{\partial_\hidden \loss(y, \pred_\param(\hidden_\horizon))}, 
\]
If no regularization is used in the definition of the regularized inverse, i.e., if we use 
\[
\dyn_{\param, t}^{-1}(\hidden_t) = (W_{\hidden\hidden}^\top W_{\hidden\hidden})^{-1}W_{\hidden \hidden}^\top( \activ^{-1}(\hidden_t) -W_{\inpt \hidden} \inpt_t  - b_\hidden),
\]
which requires the inverse of $W_{\hidden\hidden}$ to be well defined, we would get 
\[
\partial \dyn_{\param, t}^{-1}(\hidden_t) = \partial_{\hidden} \dyn_{t+1, \param}(\hidden_t)^{-1}.
\]
The updates of TP using the formula~\eqref{eq:newtarg} would then be exactly the ones of a GN update direction, i.e., 
\[
\target_t -\hidden_t = -\stepsize_\hidden \prod_{s=t+1}^{\horizon-1} \left(\partial_{\hidden} \dyn_{t+1, \param}(\hidden_t)\right)^{-\top}
{\partial_\hidden \loss(y, \pred_\param(\hidden_\horizon))}.
\]
So by considering our implementation without regularization, we can test whether the interpretation of TP as an approximate GN method is meaningful in terms of optimization convergence. As shown in Fig.~\ref{fig:regimes}, it appears that regularizing the inverses is necessary to obtain convergence, hence the interpretation of TP as GN may not be sufficient to explain why TP can converge.

%% file: appendix/d_exp_details.tex
\subsection{Initialization and hyper-parameters}
\paragraph{Initialization and data generation}
In all experiments, the weights of the RNN are initialized as random orthogonal matrices, and the biases are initialized as 0 as presented by \citet{le2015simple} and~\citet{manchev2020target}. 
For all experiments, the data was not normalized, as done by \citet{manchev2020target}. We kept a setting as similar as possible as the one of \citet{manchev2020target} to be able to compare target propagation with regularized or parameterized inverses.

\begin{table}
	\begin{center}
		\bgroup
		\def\arraystretch{2}
		\begin{tabular}{l c ccc}
			\multicolumn{1}{c}{} & \multicolumn{1}{c}{BP} &  & \multicolumn{1}{c}{TP}&  \\
			&$\gamma $ & $\gamma_h$ & $\gamma_\theta$ &$ \kappa$\\ 
			\toprule
			Temporal order problem length 60& $10^{-5}$ &$ 10^{-2} $& $10^{-1}$& $10$ \\
			Temporal order problem length 120& $10^{-5}$ &$ 10^{-2} $& $10^{-2}$& $1$ \\
			Adding problem & $10^{-3}$ & $10^{-1}$ & $10^{-1}$ & $1$ \\
			MNIST pixel by pixel & $10^{-6}$ & $10^{-4}$   & $10^{-1}$ & $1$ \\
			MNIST pixel by pixel permuted& $10^{-4}$ & $10^{-4}$   & $10^{-1}$ & $1$ \\
			CIFAR& $10^{-3}$ & $10^{-2}$   & $10^{-2}$ & $10$ \\
			FashionMNIST with GRU& $10^{-2}$ & $10^{-1}$   & $10^{-2}$ & $1$ \\
			\bottomrule
		\end{tabular}
		\egroup
	\end{center}
	\caption{Hyper-parameters chosen for Fig.~\ref{fig:synth} and~\ref{fig:exp_mnist}. \label{tab:hyper_param}}
\end{table}

\paragraph{Hyper-parameters}
In the synthetic tasks, for BP we used a momentum  of $0.9$ with Nesterov accelerated gradient scheme as done by~\citet{manchev2020target}. Otherwise, we did not use any momentum for the experiment on MNIST pixel by pixel presented in the main paper. 
The learning rates of BP and the parameters of TP were found by a grid-search on a $\log_{10}$ basis and are presented in Table~\ref{tab:hyper_param}.
We did not add a regularization term in the training of the RNNs.

For the Fig.~\ref{fig:regimes}, we used batch sizes of size 512 and performed a grid search for the stepsizes of BP and for the stepsizes $\stepsize_\hidden$ of TP while keeping the same regularization $\reg$ and stepsize $\stepsize_\param$ to the parameters found for the length 784. 

\paragraph{Software} We used Python 3.8 and PyTorch 1.6. The RNN was coded using the cuDNN implementation available in PyTorch that is highly optimized for computing forward passes on the network or gradient back-propagation.

\paragraph{Hardware}
All experiments were performed on GPUs using Nvidia GeForce GTX 1080 Ti (12G memory). Each experiment only used one gpu at a time (clock speed 1.5 Ghz).

\paragraph{Time evaluation}
On our GPU, we observed that for the MNIST pixel by pixel experiment, 200 iterations (each iteration considering 16 samples)  were taking approximately 60s for BP and 800s for TP. Note that with larger batch-sizes the cost of the regularized inversion would be  amortized by the fact that more samples are treated simultaneously. We kept the setting of Manchev and Spartling~\citep{manchev2020target} for ease of comparison.

\subsection{Additional experiments}
\paragraph{The overhead of TP can be worth its performance}
To account for the additional cost of inversion for each mini-batch, we consider the convergence of the algorithms in time rather than in iterations. We found that, on average, 1 iteration of BP takes approximately 13 times less time than one iteration of TP in our implementation (note that BP benefits from highly optimized implementations for GPU machines, and TP could potentially also benefit from the same optimized implementations). Therefore we ran BP for 13 times more iterations than TP and multiplied the number of iterations by the approximate time needed for each iteration for all algorithms. In the right panel of Fig.~\ref{fig:mnist_time_dtp}, we observe that in time too, TP performs better than BP, which stays stuck at an accuracy of approximately 22

\paragraph{Regularized inverses outperform parameterized inverses}
We evaluate the impact of using regularized inverses as opposed to parameterized inverse and linearized propagation as opposed to finite-difference-based propagation.
The variant of target propagation with parameterized inverse and finite-difference propagation corresponds to the approach of~\citet{lee2015difference} recently implemented by~\citet{manchev2020target} and referred to in the figure above as \textbf{DTP-PI}.
The variant of target propagation with regularized inverse and finite-difference propagation is referred to in the figure above as \textbf{DTP-RI}. Recall that our approach involves regularized inverses and linearized propagation, referred as \textbf{TP}. 
In Fig.~\ref{fig:mnist_time_dtp}, we observe that both \TP and DTP-RI outperform DTP-PI, demonstrating the benefits of using regularized inverses. 
On the other hand, both \TP and DTP-RI perform on par  overall, with the former being slightly better for the given parameters.

\paragraph{Target propagation is robust to the choice of the target stepsize $\stepsize_\hidden$}
In Fig.~\ref{fig:heatmap}, we observed how the convergence could be affected by the choice of the regularization and the step-size $\stepsize_\param$. 
In the left panel of Fig.~\ref{fig:add_exp}, we observe that varying $\stepsize_\hidden$ does not lead to significant changes in the convergence behavior.

\paragraph{Target propagation does not benefit from momentum techniques}
Numerous methods have been proposed to enhance the performance of a classical stochastic gradient descent by using, e.g., a momentum term akin to Nesterov's accelerated gradient formula ~\citep{sutskever2013importance}. 
Since TP also produces a priori a descent direction, we can wonder whether an additional momentum provides faster convergence. In the right panel of Fig.~\ref{fig:add_exp}, we observe that adding a momentum on TP is possible but does not seem to provide significantly faster convergence while being less stable. On the other hand, on this same figure, the momentum seems to help the gradient descent. Our preliminary experiments using Adam with TP did not conclude; namely, we were not able to obtain a convergence similar to the one illustrated in Fig.~\eqref{fig:exp_mnist} that simply used Algo.~\ref{algo:linearized_target_prop}. We leave for future work the implementation of appropriate adaptive stepsizes strategies for TP.

\begin{figure}
	\begin{center}
		\includegraphics[width=0.35\linewidth]{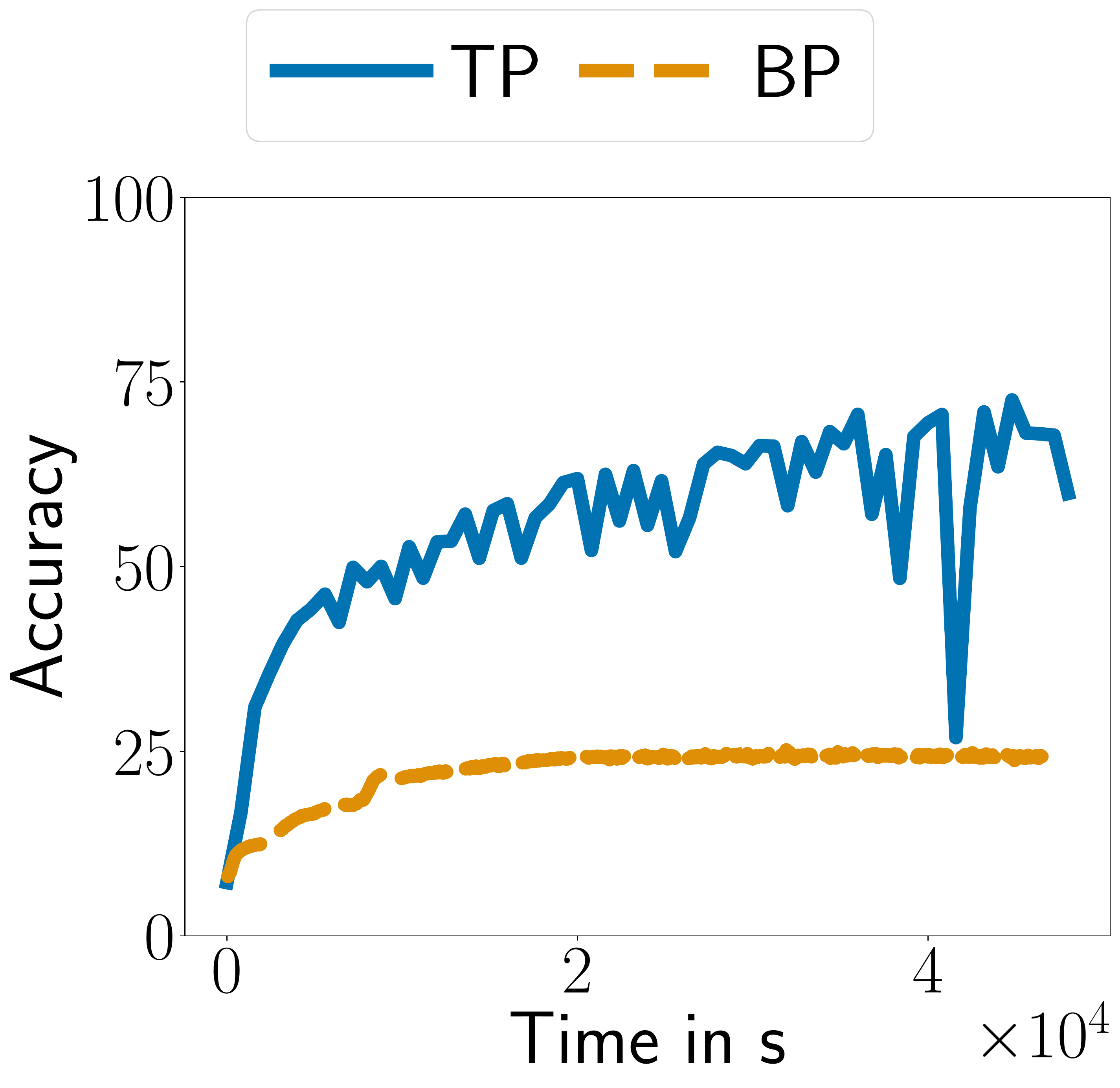}~\hspace{5em}
		\includegraphics[width=0.35\linewidth]{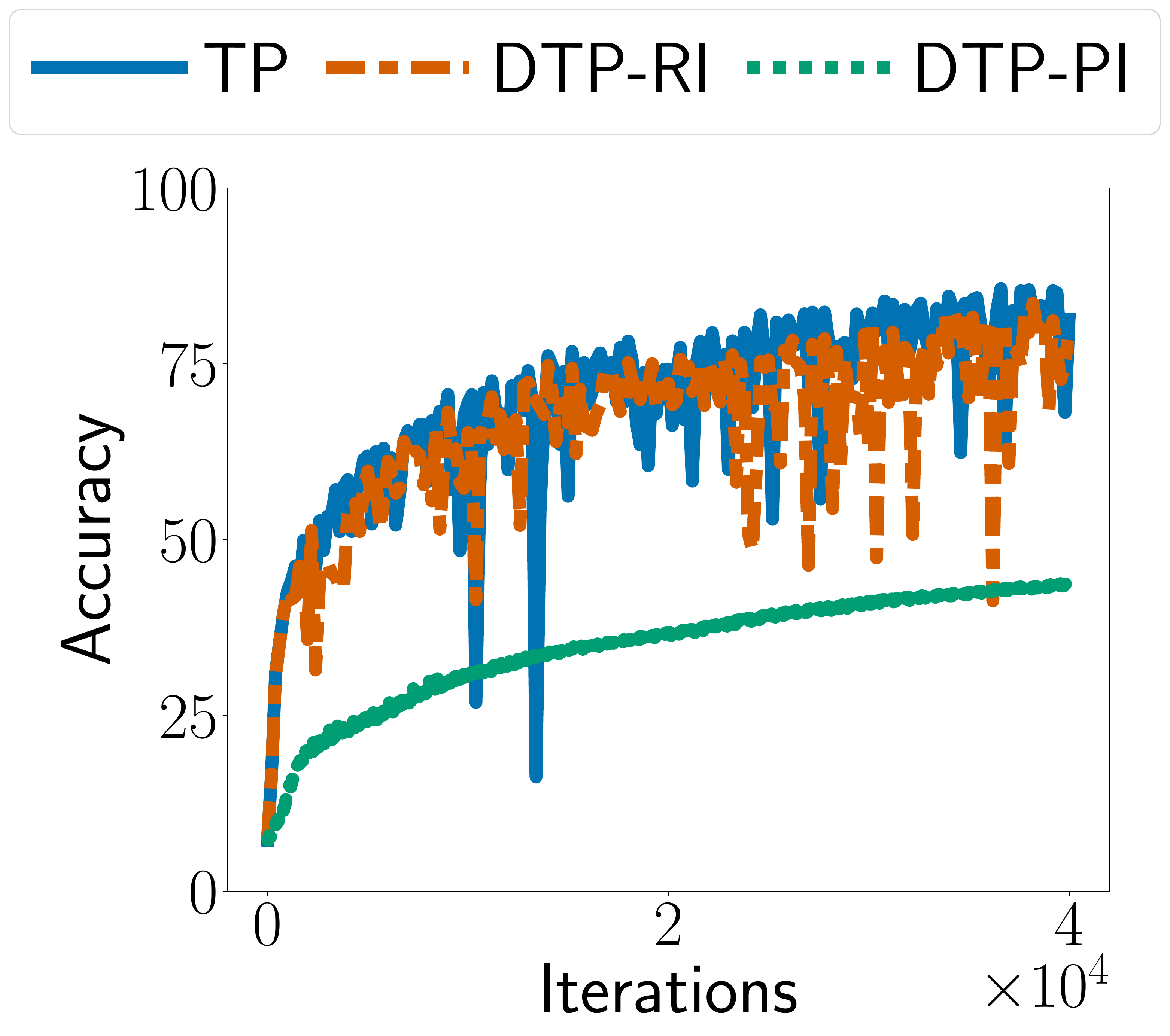}
		\caption{Left: MNIST in time\label{fig:mnist_time_dtp}. Right: Comparison of different implementations of TP.}
	\end{center}
\end{figure}

\begin{figure}
	\begin{center}
		\includegraphics[width=0.5\linewidth]{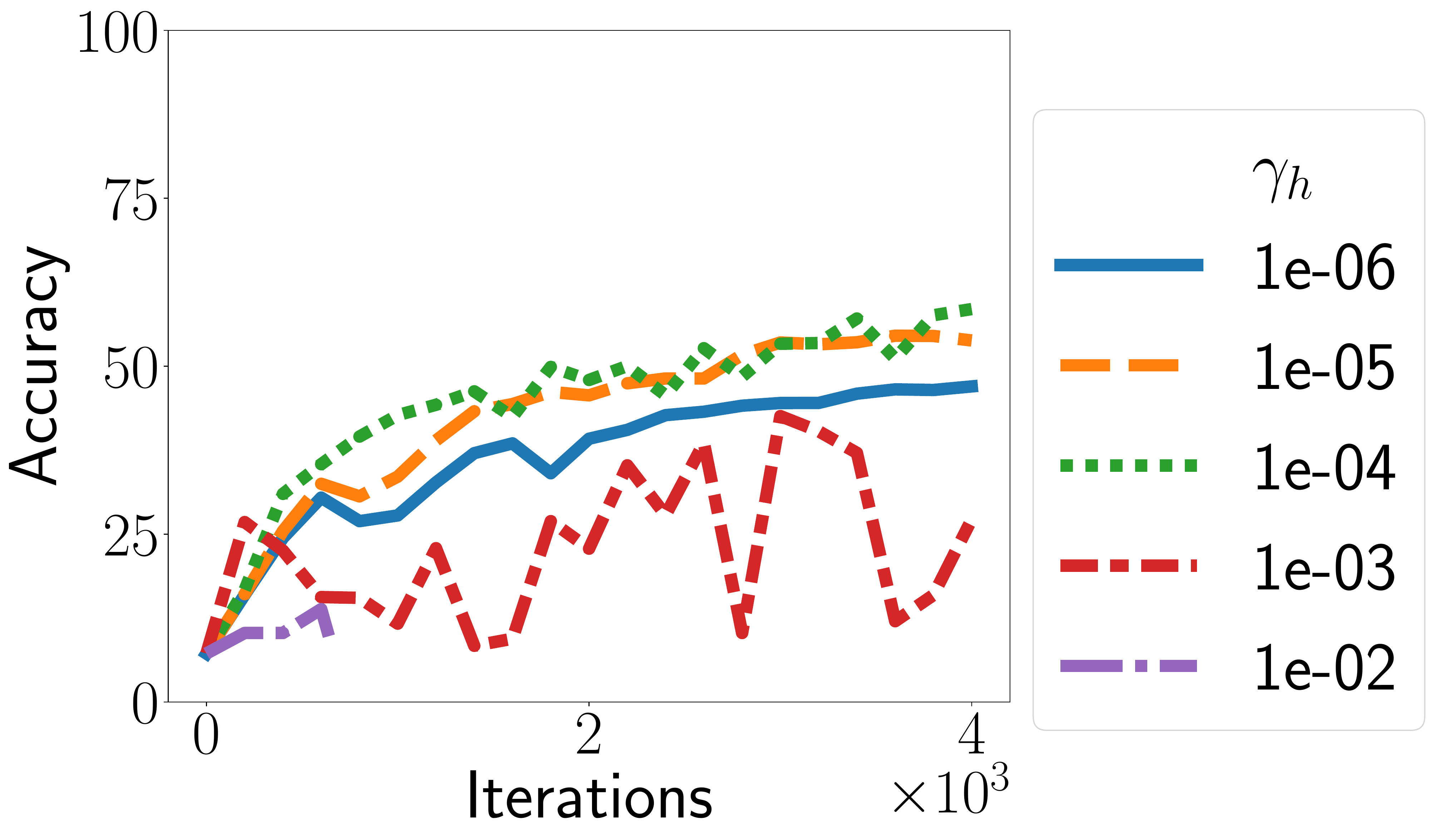}~
		\includegraphics[width=0.35\linewidth]{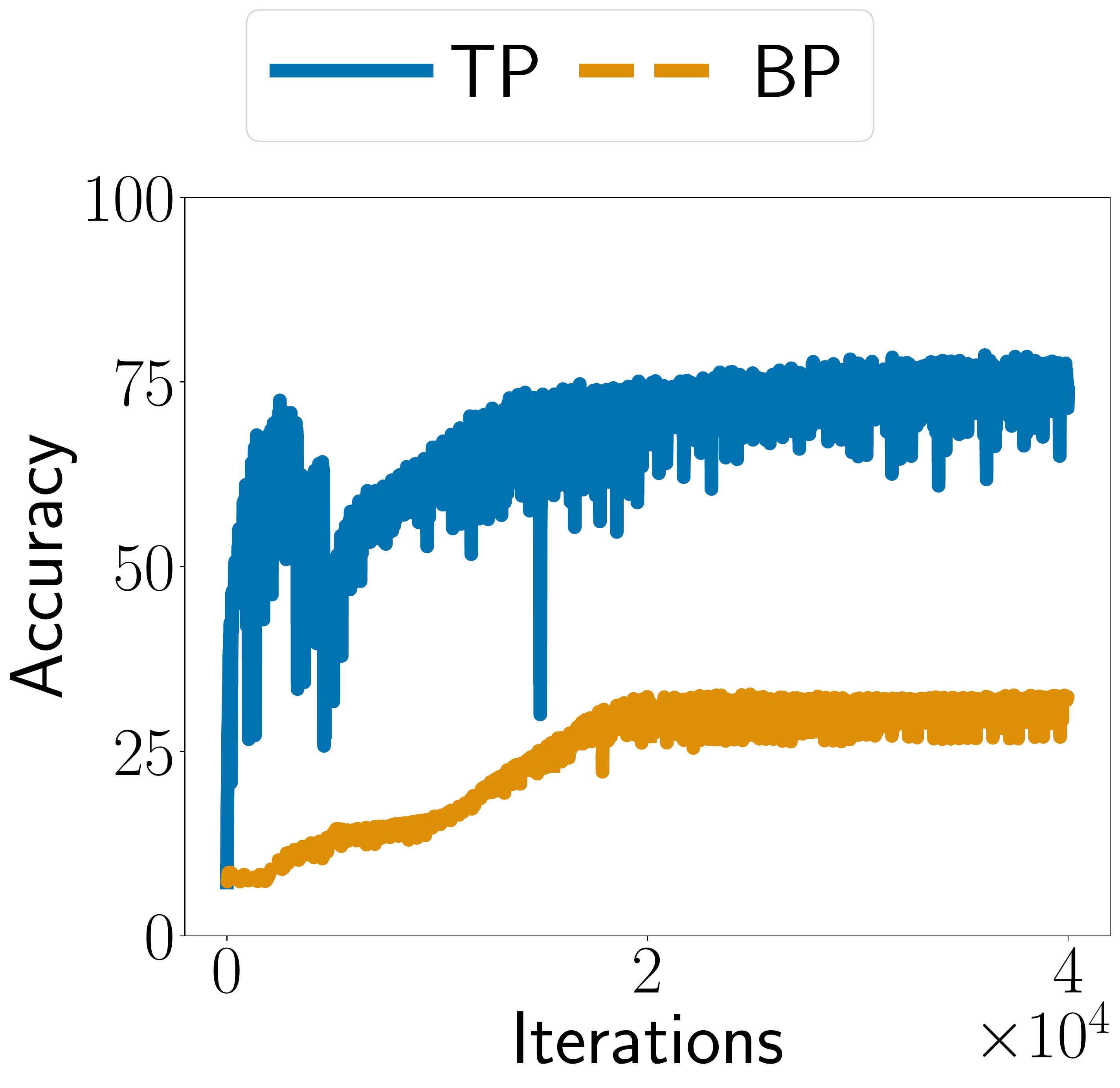}
		\caption{Left: MNIST for varying $\stepsize_\hidden$ and fixed $\stepsize_\param=1$, $\reg=1$\label{fig:add_exp}. Right: MNIST with momentum.}
	\end{center}
\end{figure}

%% file: target_prop.bbl
\begin{thebibliography}{42}
\providecommand{\natexlab}[1]{#1}
\providecommand{\url}[1]{\texttt{#1}}
\expandafter\ifx\csname urlstyle\endcsname\relax
  \providecommand{\doi}[1]{doi: #1}\else
  \providecommand{\doi}{doi: \begingroup \urlstyle{rm}\Url}\fi

\bibitem[Ahmad et~al.(2020)Ahmad, van Gerven, and Ambrogioni]{ahmad2020gait}
Nasir Ahmad, Marcel~A van Gerven, and Luca Ambrogioni.
\newblock Gait-prop: A biologically plausible learning rule derived from
  backpropagation of error.
\newblock \emph{Advances in Neural Information Processing Systems}, 33, 2020.

\bibitem[Arjovsky et~al.(2016)Arjovsky, Shah, and Bengio]{arjovsky2016unitary}
Martin Arjovsky, Amar Shah, and Yoshua Bengio.
\newblock Unitary evolution recurrent neural networks.
\newblock In \emph{Proceedings of the 33rd International Conference on Machine
  Learning}, 2016.

\bibitem[Bengio(2014)]{bengio2014auto}
Yoshua Bengio.
\newblock How auto-encoders could provide credit assignment in deep networks
  via target propagation.
\newblock \emph{arXiv preprint arXiv:1407.7906}, 2014.

\bibitem[Bengio(2020)]{bengio2020deriving}
Yoshua Bengio.
\newblock Deriving differential target propagation from iterating approximate
  inverses.
\newblock \emph{arXiv preprint arXiv:2007.15139}, 2020.

\bibitem[Bengio and Frasconi(1995)]{bengio1995diffusion}
Yoshua Bengio and Paolo Frasconi.
\newblock Diffusion of context and credit information in markovian models.
\newblock \emph{Journal of Artificial Intelligence Research}, 3:\penalty0
  249--270, 1995.

\bibitem[Bengio et~al.(1994)Bengio, Simard, and Frasconi]{bengio1994learning}
Yoshua Bengio, Patrice Simard, and Paolo Frasconi.
\newblock Learning long-term dependencies with gradient descent is difficult.
\newblock \emph{IEEE transactions on neural networks}, 5\penalty0 (2):\penalty0
  157--166, 1994.

\bibitem[Bengio et~al.(2013)Bengio, L{\'e}onard, and
  Courville]{bengio2013estimating}
Yoshua Bengio, Nicholas L{\'e}onard, and Aaron Courville.
\newblock Estimating or propagating gradients through stochastic neurons for
  conditional computation.
\newblock \emph{arXiv preprint arXiv:1308.3432}, 2013.

\bibitem[Carreira-Perpinan and Wang(2014)]{carreira2014distributed}
Miguel Carreira-Perpinan and Weiran Wang.
\newblock Distributed optimization of deeply nested systems.
\newblock In \emph{Proceedings of the 17th International Conference on
  Artificial Intelligence and Statistics}, 2014.

\bibitem[Cho et~al.(2014)Cho, Van~Merri{\"e}nboer, Gulcehre, Bahdanau,
  Bougares, Schwenk, and Bengio]{cho2014learning}
Kyunghyun Cho, Bart Van~Merri{\"e}nboer, Caglar Gulcehre, Dzmitry Bahdanau,
  Fethi Bougares, Holger Schwenk, and Yoshua Bengio.
\newblock Learning phrase representations using rnn encoder-decoder for
  statistical machine translation.
\newblock \emph{arXiv preprint arXiv:1406.1078}, 2014.

\bibitem[Czarnecki et~al.(2017)Czarnecki, {\'{S}}wirszcz, Jaderberg, Osindero,
  Vinyals, and Kavukcuoglu]{czarnecki2017understanding}
Wojciech~Marian Czarnecki, Grzegorz {\'{S}}wirszcz, Max Jaderberg, Simon
  Osindero, Oriol Vinyals, and Koray Kavukcuoglu.
\newblock Understanding synthetic gradients and decoupled neural interfaces.
\newblock In \emph{Proceedings of the 34th International Conference on Machine
  Learning}, 2017.

\bibitem[Dalm et~al.(2021)Dalm, Ahmad, Ambrogioni, and van
  Gerven]{dalm2021scaling}
Sander Dalm, Nasir Ahmad, Luca Ambrogioni, and Marcel van Gerven.
\newblock Scaling up learning with gait-prop.
\newblock \emph{arXiv preprint arXiv:2102.11598}, 2021.

\bibitem[Devolder et~al.(2014)Devolder, Glineur, and
  Nesterov]{devolder2014first}
Olivier Devolder, Fran{\c{c}}ois Glineur, and Yurii Nesterov.
\newblock First-order methods of smooth convex optimization with inexact
  oracle.
\newblock \emph{Mathematical Programming}, 146\penalty0 (1-2):\penalty0 37--75,
  2014.

\bibitem[Frerix et~al.(2018)Frerix, Möllenhoff, Moeller, and
  Cremers]{frerix2018proximal}
Thomas Frerix, Thomas Möllenhoff, Michael Moeller, and Daniel Cremers.
\newblock Proximal backpropagation.
\newblock In \emph{Proceedings of the 6th International Conference on Learning
  Representations}, 2018.

\bibitem[Goodfellow et~al.(2016)Goodfellow, Bengio, and
  Courville]{goodfellow2016deep}
Ian Goodfellow, Yoshua Bengio, and Aaron Courville.
\newblock \emph{Deep Learning}.
\newblock The MIT Press, 2016.

\bibitem[Gotmare et~al.(2018)Gotmare, Thomas, Brea, and
  Jaggi]{gotmare2018decoupling}
Akhilesh Gotmare, Valentin Thomas, Johanni Brea, and Martin Jaggi.
\newblock Decoupling backpropagation using constrained optimization methods.
\newblock In \emph{Credit Assignment in Deep Learning and Reinforcement
  Learning Workshop (ICML 2018 ECA)}, 2018.

\bibitem[Helfrich et~al.(2018)Helfrich, Willmott, and
  Ye]{helfrich2018orthogonal}
Kyle Helfrich, Devin Willmott, and Qiang Ye.
\newblock Orthogonal recurrent neural networks with scaled cayley transform.
\newblock In \emph{Proceedings of the 35th International Conference on Machine
  Learning}, 2018.

\bibitem[Hochreiter(1998)]{hochreiter1998vanishing}
Sepp Hochreiter.
\newblock The vanishing gradient problem during learning recurrent neural nets
  and problem solutions.
\newblock \emph{International Journal of Uncertainty, Fuzziness and
  Knowledge-Based Systems}, 6\penalty0 (02):\penalty0 107--116, 1998.

\bibitem[Hochreiter and Schmidhuber(1997)]{hochreiter1997long}
Sepp Hochreiter and J{\"u}rgen Schmidhuber.
\newblock Long short-term memory.
\newblock \emph{Neural computation}, 9\penalty0 (8):\penalty0 1735--1780, 1997.

\bibitem[Jaderberg et~al.(2017)Jaderberg, Czarnecki, Osindero, Vinyals, Graves,
  Silver, and Kavukcuoglu]{jaderberg2017decoupled}
Max Jaderberg, Wojciech~Marian Czarnecki, Simon Osindero, Oriol Vinyals, Alex
  Graves, David Silver, and Koray Kavukcuoglu.
\newblock Decoupled neural interfaces using synthetic gradients.
\newblock In \emph{Proceedings of the 34th International Conference on Machine
  Learning}, 2017.

\bibitem[Krizhevsky(2009)]{krizhevsky2009learning}
Alex Krizhevsky.
\newblock Learning multiple layers of features from tiny images.
\newblock Technical report, University of Toronto, 2009.

\bibitem[Le et~al.(2015)Le, Jaitly, and Hinton]{le2015simple}
Quoc~V Le, Navdeep Jaitly, and Geoffrey~E Hinton.
\newblock A simple way to initialize recurrent networks of rectified linear
  units.
\newblock \emph{arXiv preprint arXiv:1504.00941}, 2015.

\bibitem[Le~Cun(1986)]{lecun1986learning}
Yann Le~Cun.
\newblock Learning process in an asymmetric threshold network.
\newblock In \emph{Disordered systems and biological organization}. Springer,
  1986.

\bibitem[Le~Cun et~al.(1988)Le~Cun, Galland, and Hinton]{lecun1989gemini}
Yann Le~Cun, Conrad~C Galland, and Geoffrey~E Hinton.
\newblock {GEMINI}: gradient estimation through matrix inversion after noise
  injection.
\newblock In \emph{Advances in Neural Information Processing Systems 1}, 1988.

\bibitem[LeCun and Cortes(1998)]{lecun2010mnist}
Yann LeCun and Corinna Cortes.
\newblock {MNIST} handwritten digit database.
\newblock http://yann.lecun.com/exdb/mnist/, 1998.

\bibitem[Lee et~al.(2015)Lee, Zhang, Fischer, and Bengio]{lee2015difference}
Dong-Hyun Lee, Saizheng Zhang, Asja Fischer, and Yoshua Bengio.
\newblock Difference target propagation.
\newblock In \emph{Machine Learning and Knowledge Discovery in Databases}.
  Springer, 2015.

\bibitem[Lezcano-Casado and Mart{\i}nez-Rubio(2019)]{lezcano2019cheap}
Mario Lezcano-Casado and David Mart{\i}nez-Rubio.
\newblock Cheap orthogonal constraints in neural networks: A simple
  parametrization of the orthogonal and unitary group.
\newblock In \emph{Proceedings of the 36th International Conference on Machine
  Learning}, 2019.

\bibitem[Manchev and Spratling(2020)]{manchev2020target}
Nikolay Manchev and Michael Spratling.
\newblock Target propagation in recurrent neural networks.
\newblock \emph{Journal of Machine Learning Research}, 21\penalty0
  (7):\penalty0 1--33, 2020.

\bibitem[Meulemans et~al.(2020)Meulemans, Carzaniga, Suykens, Sacramento, and
  Grewe]{meulemans2020theoretical}
Alexander Meulemans, Francesco Carzaniga, Johan Suykens, Jo\~{a}o Sacramento,
  and Benjamin~F. Grewe.
\newblock A theoretical framework for target propagation.
\newblock In \emph{Advances in Neural Information Processing Systems 33}, 2020.

\bibitem[Meulemans et~al.(2021)Meulemans, Farinha, Ord{\'o}{\~n}ez, Aceituno,
  Sacramento, and Grewe]{meulemans2021credit}
Alexander Meulemans, Matilde~Tristany Farinha, Javier~Garc{\'\i}a
  Ord{\'o}{\~n}ez, Pau~Vilimelis Aceituno, Jo{\~a}o Sacramento, and Benjamin~F
  Grewe.
\newblock Credit assignment in neural networks through deep feedback control.
\newblock \emph{arXiv preprint arXiv:2106.07887}, 2021.

\bibitem[Mirowski and LeCun(2009)]{mirowski2009dynamic}
Piotr Mirowski and Yann LeCun.
\newblock Dynamic factor graphs for time series modeling.
\newblock In \emph{Machine Learning and Knowledge Discovery in Databases}.
  Springer, 2009.

\bibitem[Pascanu et~al.(2012)Pascanu, Mikolov, and
  Bengio]{pascanu2012understanding}
Razvan Pascanu, Tomas Mikolov, and Yoshua Bengio.
\newblock Understanding the exploding gradient problem.
\newblock \emph{arXiv preprint arXiv:1211.5063}, 2012.

\bibitem[Paszke et~al.(2019)Paszke, Gross, Massa, Lerer, Bradbury, Chanan,
  Killeen, Lin, Gimelshein, Antiga, Desmaison, Kopf, Yang, DeVito, Raison,
  Tejani, Chilamkurthy, Steiner, Fang, Bai, and Chintala]{paszke2017automatic}
Adam Paszke, Sam Gross, Francisco Massa, Adam Lerer, James Bradbury, Gregory
  Chanan, Trevor Killeen, Zeming Lin, Natalia Gimelshein, Luca Antiga, Alban
  Desmaison, Andreas Kopf, Edward Yang, Zachary DeVito, Martin Raison, Alykhan
  Tejani, Sasank Chilamkurthy, Benoit Steiner, Lu~Fang, Junjie Bai, and Soumith
  Chintala.
\newblock Pytorch: An imperative style, high-performance deep learning library.
\newblock In \emph{Advances in Neural Information Processing Systems 32}, 2019.

\bibitem[Rohwer(1989)]{rohwer1990moving}
Richard Rohwer.
\newblock The ``moving targets" training algorithm.
\newblock In \emph{Advances in Neural Information Processing Systems 2}, 1989.

\bibitem[Roulet et~al.(2019)Roulet, Srinivasa, Drusvyatskiy, and
  Harchaoui]{roulet2019iterative}
Vincent Roulet, Siddhartha Srinivasa, Dmitriy Drusvyatskiy, and Zaid Harchaoui.
\newblock Iterative linearized control: stable algorithms and complexity
  guarantees.
\newblock In \emph{Proceedings of the 36th International Conference on Machine
  Learning}, 2019.

\bibitem[Rumelhart et~al.(1986)Rumelhart, Hinton, and
  Williams]{rumelhart1985learning}
D.~E. Rumelhart, G.~E. Hinton, and R.~J. Williams.
\newblock \emph{Learning Internal Representations by Error Propagation}.
\newblock MIT Press, Cambridge, MA, USA, 1986.

\bibitem[Schmidhuber(1992)]{schmidhuber1992learning}
J{\"u}rgen Schmidhuber.
\newblock Learning complex, extended sequences using the principle of history
  compression.
\newblock \emph{Neural Computation}, 4\penalty0 (2):\penalty0 234--242, 1992.

\bibitem[Sutskever et~al.(2011)Sutskever, Martens, and
  Hinton]{sutskever2011generating}
Ilya Sutskever, James Martens, and Geoffrey~E Hinton.
\newblock Generating text with recurrent neural networks.
\newblock In \emph{Proceedings of the 28th International Conference on Machine
  Learning}, 2011.

\bibitem[Sutskever et~al.(2013)Sutskever, Martens, Dahl, and
  Hinton]{sutskever2013importance}
Ilya Sutskever, James Martens, George Dahl, and Geoffrey Hinton.
\newblock On the importance of initialization and momentum in deep learning.
\newblock In \emph{Proceedings of the 30th International Conference on Machine
  Learning}, 2013.

\bibitem[Taylor et~al.(2016)Taylor, Burmeister, Xu, Singh, Patel, and
  Goldstein]{taylor2016training}
Gavin Taylor, Ryan Burmeister, Zheng Xu, Bharat Singh, Ankit Patel, and Tom
  Goldstein.
\newblock Training neural networks without gradients: A scalable {ADMM}
  approach.
\newblock In \emph{Proceedings of the 33rd International Conference on Machine
  Learning}, 2016.

\bibitem[Werbos(1994)]{werbos1994roots}
Paul Werbos.
\newblock \emph{The Roots of Backpropagation: From Ordered Derivatives to
  Neural Networks and Political Forecasting}.
\newblock Wiley-Interscience, 1994.

\bibitem[Wiseman et~al.(2017)Wiseman, Chopra, Ranzato, Szlam, Sun, Chintala,
  and Vasilache]{wiseman2017training}
Sam Wiseman, Sumit Chopra, Marc-Aurelio Ranzato, Arthur Szlam, Ruoyu Sun,
  Soumith Chintala, and Nicolas Vasilache.
\newblock Training language models using target-propagation.
\newblock \emph{arXiv preprint arXiv:1702.04770}, 2017.

\bibitem[Xiao et~al.(2017)Xiao, Rasul, and Vollgraf]{xiao2017fashion}
Han Xiao, Kashif Rasul, and Roland Vollgraf.
\newblock Fashion-mnist: a novel image dataset for benchmarking machine
  learning algorithms, 2017.
\newblock URL \url{https://github.com/zalandoresearch/fashion-mnist}.

\end{thebibliography}
